\documentclass[twoside]{article}

\usepackage{amsmath,amsfonts,bm}

\newcommand{\unif}{\mathrm{Unif}}

\def\eqref#1{(\ref{#1})}
\def\Eqref#1{Equation~\ref{#1}}

\def\1{\bm{1}}

\def\rd{{\textnormal{d}}}

\def\vzero{{\bm{0}}}

\def\valpha{{\bm{\alpha}}}

\def\veps{{\bm{\epsilon}}}
\def\vmu{{\bm{\mu}}}

\def\vb{{\bm{b}}}

\def\vf{{\bm{f}}}

\def\vr{{\bm{r}}}

\def\vv{{\bm{v}}}
\def\vw{{\bm{w}}}
\def\vx{{\bm{x}}}

\def\vmu{{\bm{\mu}}}

\def\mA{{\bm{A}}}
\def\mB{{\bm{B}}}

\def\mI{{\bm{I}}}

\def\mX{{\bm{X}}}
\def\mY{{\bm{Y}}}

\def\mSigma{{\bm{\Sigma}}}
\def\msigma{{\bm{\sigma}}}

\DeclareMathAlphabet{\mathsfit}{\encodingdefault}{\sfdefault}{m}{sl}
\SetMathAlphabet{\mathsfit}{bold}{\encodingdefault}{\sfdefault}{bx}{n}

\def\gA{{\mathcal{A}}}

\def\gD{{\mathcal{D}}}

\def\gL{{\mathcal{L}}}
\def\gM{{\mathcal{M}}}
\def\gN{{\mathcal{N}}}
\def\gO{{\mathcal{O}}}
\def\gP{{\mathcal{P}}}

\def\gT{{\mathcal{T}}}

\def\sP{{\mathbb{P}}}

\def\sR{{\mathbb{R}}}

\def\sZ{{\mathbb{Z}}}

\newcommand{\E}{\mathbb{E}}

\newcommand{\R}{\mathbb{R}}

\newcommand{\Var}{\mathrm{Var}}

\newcommand{\TV}{\mathrm{TV}}

\DeclareMathOperator{\tr}{tr}
\DeclareMathOperator{\diag}{diag}

\usepackage[accepted]{aistats2024}

\usepackage{natbib}

\usepackage[T1]{fontenc}
\usepackage{hyperref}
\usepackage{url}
\usepackage{color, xcolor}
\usepackage{graphicx}
\usepackage{caption, subcaption}
\usepackage{enumitem}
\usepackage{xcolor}
\usepackage{amsthm,amssymb}
\newtheorem{theorem}{Theorem}[section]
\newtheorem{corollary}[theorem]{Corollary}
\newtheorem{lemma}[theorem]{Lemma}

\newtheorem{remark}[theorem]{Remark}
\newtheorem{definition}{Definition}[section]
\newtheorem{proposition}[theorem]{Proposition}

\numberwithin{equation}{section}

\begin{document}

\twocolumn[

\aistatstitle{Understanding the Generalization Benefits of Late Learning Rate Decay}

\aistatsauthor{ Yinuo Ren \And Chao Ma \And Lexing Ying }

\aistatsaddress{ Stanford University \And Stanford University \And Stanford University } ]

\begin{abstract}
    Why do neural networks trained with large learning rates for a longer time
    often lead to better generalization? In this paper, we delve into this
    question by examining the relation between training and testing loss in
    neural networks. Through visualization of these losses, we note that the
    training trajectory with a large learning rate navigates through the minima
    manifold of the training loss, finally nearing the neighborhood of the
    testing loss minimum. Motivated by these findings, we introduce a nonlinear
    model whose loss landscapes mirror those observed for real neural networks.
    Upon investigating the training process using SGD on our model, we
    demonstrate that an extended phase with a large learning rate steers our
    model towards the minimum norm solution of the training loss, which may
    achieve near-optimal generalization, thereby affirming the empirically
    observed benefits of late learning rate decay.
\end{abstract}

\section{INTRODUCTION}
During the training of deep neural networks, one of the challenges faced by optimization algorithms arises from the intricate misalignment between the training and testing losses. This discrepancy becomes particularly evident in overparameterized settings, in which case merely minimizing the training loss does not necessarily translate to desirable testing performance. Nonetheless, in practical scenarios, neural networks often demonstrate impressive generalization when trained using stochastic gradient descent (SGD)~\citep{allen2019can,kleinberg2018alternative,pesme2021implicit}. This relative ease of training can be attributed, at least in part, to the development and adoption of several techniques over the years, such as normalization~\citep{salimans2016weight,ba2016layer}, adaptive optimization~\citep{duchi2011adaptive,kingma2014adam}, and learning rate schemes~\citep{smith2017cyclical,loshchilov2017decoupled}. 

It is a widely accepted observation that implementing late learning rate decay, \emph{i.e.} maintaining a large learning rate with SGD for an extended period even after the stabilization of the training loss, can enhance generalization performance~\citep{li2019towards,wu2020direction,wang2021large,beugnot2022benefits}.
Numerous theoretical studies have explored and interpreted related phenomena are rich, and we direct readers to Section~\ref{sec:related work} for a brief review. For example,~\citet{wu2020direction} studies linear regression and argues that the ``directional bias'' of SGD with a large learning rate boosts final testing performance, while \citet{li2019towards} emphasizes the mismatch between learnability and generalizability can explain the necessity of an initial large learning rate. However, there is still a noticeable gap in the literature regarding the interplay between training and testing losses and their landscape formations.

Our study is motivated by specific observations from our visualization of the training and testing loss landscapes of neural networks, as depicted in Figure~\ref{fig:loss_landscape}:
\vspace{-0.5em}
\begin{itemize}[leftmargin=*]
    \item The training loss landscape displays a \emph{minima manifold} characterized by open level sets, while the testing loss landscape presents an isolated minimum with closed level sets.  
    \item With a large learning rate, the training trajectory navigates through the minima manifold of the training loss towards the neighborhood of the testing loss minimum.
\end{itemize}
\vspace{-0.5em}
Several previous works contend that the traversal of the minima manifold is driven by the flatness of the training loss landscape~\citep{wu2018sgd,mulayoff2021implicit,nacson2022implicit}. Yet, the question lingers as to why the training loss landscape is inherently flatter around the testing loss minimum.

In this work, we scrutinize the relation between training and testing loss landscapes with neural networks as nonlinear overparameterized models and the implication for training behaviors. We propose a simple nonlinear model whose loss landscapes mirror our visualizations from neural networks. Our model can be interpreted as follows: Starting with an overparametrized linear regression model, the testing loss, as an expectation of quadratic functions, yields an isolated minimum, while the null space of the training data produces a minima manifold. Then, a transformation motivated by the depth of the neural network is applied to both losses within the parameter space, resulting in non-quadratic landscapes exhibiting varying flatness. 

We then study the training process of our model using SGD via a continuous-time analysis. Our findings suggest that this process can be divided into three phases: (I) an initial phase with a large learning rate, (II) an extended phase maintaining the large learning rate, and (III) a final phase with a decayed learning rate. We prove that with high probability, Phase II propels the model towards the minimum $L^2$-norm solution of the training loss, which has long been believed to be the near-optimal solution for overparametrized models~\citep{wu2020optimal,bartlett2020benign}, thereby affirming the empirically observed benefits of late learning rate decay.

\subsection{Contribution}
Our main contributions in this paper are summarized as follows:
\vspace{-0.5em}
\begin{itemize}[leftmargin=*]
    \item Through experiments, we empirically demonstrate the generalization advantages of late learning rate decay. Further, we offer visualizations of the training and testing loss landscapes of neural networks, illustrating the interrelation between these two loss landscapes.

    \item We introduce a nonlinear overparameterized model that recovers the loss landscape behaviors observed in real neural networks. Our insights suggest that the flatness of the training loss landscape near the testing loss minimum is intrinsically linked to the depth of the neural networks.

    \item We systematically dissect the training process of our model into three phases and show that extended training using a large learning rate helps find the minimum $L^2$-norm solution of the training loss and thus corroborates the provable benefits of late learning rate decay.
\end{itemize}
\vspace{-0.5em}

\subsection{Related Works}\label{sec:related work}
\paragraph*{Implicit Regularization.} 

The implicit regularization effect of optimization with SGD has been studied extensively in previous works~\citep{mandt2016variational,hoffer2017train,kleinberg2018alternative}. Many works argue that SGD picks flat minima~\citep{keskar2016large,du2019gradient,wu2022alignment}, which boosts the generalization performance~\citep{hochreiter1997flat,zhou2020towards}. To understand the behavior of SGD, several mathematical models have been proposed and studied, including the stochastic differential equations (SDEs)~\citep{li2017stochastic,li2019stochastic,li2021validity,mori2022power}, and Langevin dynamics~\citep{welling2011bayesian,raginsky2017non,zhang2017hitting,chen2020stationary}. Recent works including~\citep{blanc2020implicit,pesme2021implicit,haochen2021shape,damian2021label,even2023s} adopts the \emph{diagonal linear networks} to study implicit regularization.

\paragraph*{SGD Scheduling.} 

Another line of the research on SGD attempts to gain a deeper understanding of how the choice and scheduling of the learning rate affects the performance of SGD~\citep{smith2017bayesian,jastrzkebski2018width,wu2018sgd,li2020reconciling,lyu2019gradient,mulayoff2021implicit,nacson2022implicit,li2022fast}.
Large learning rates are shown to be beneficial for generalization both empirically and theoretically~\citep{li2019towards,wu2020direction,wang2021large,andriushchenko2022sgd}. Intricate designs of learning rate schedules have also been proved to achieve faster convergence rate for gradient descent~\citep{smith2017cyclical,agarwal2021acceleration,grimmer2023provably}. The study of the effect of large learning rates is also closely related to the topic of the \emph{Edge of Stability (EoS)} phenomenon~\citep{cohen2021gradient,arora2022understanding,damian2022self,zhu2022understanding,beugnot2022benefits,ma2022beyond,chen2022gradient}, and \emph{grokking}~\citep{power2022grokking,liu2022towards,vzunkovivc2022grokking}.

\section{MOTIVATING EMPIRICAL OBSERVATIONS}

In this section, we present experiment results to observe the training behaviors in neural networks. We further visualize the training and testing loss landscapes obtained by performing PCA to the parameters on the training trajectories and discuss its implications on the training process.

\subsection{Observations from Training Behaviors}

\begin{figure*}[!htb]
    \centering
    \begin{subfigure}[b]{0.32\textwidth}
        \includegraphics[width=\textwidth]{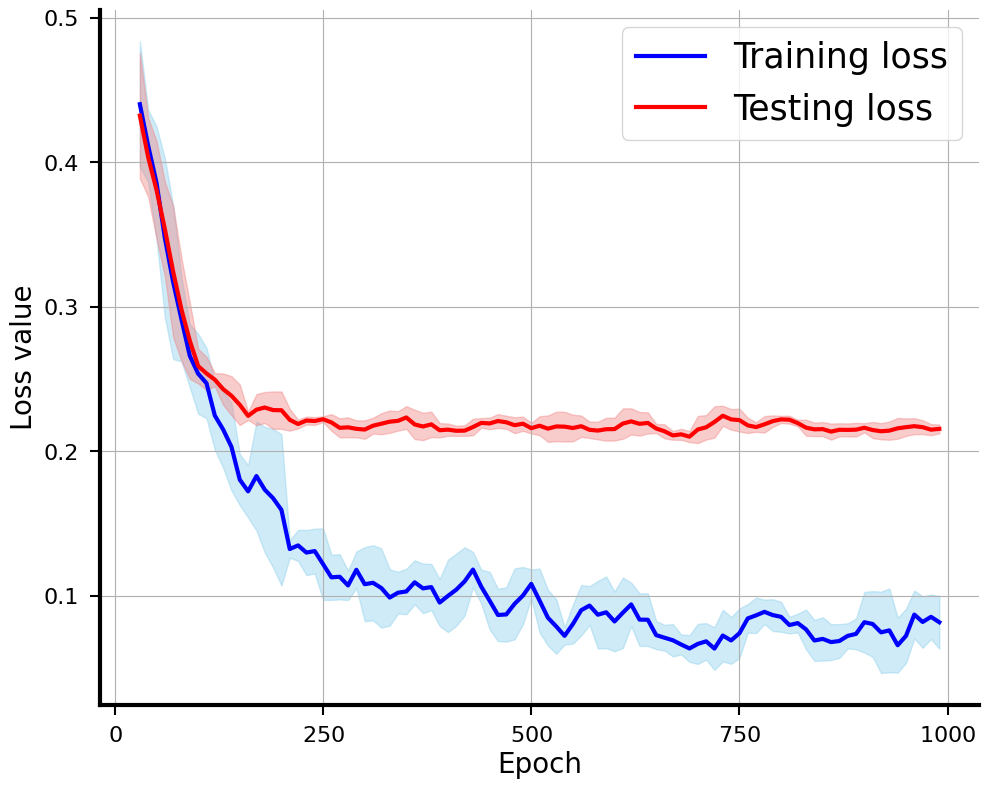}
        \caption{Learning curves of initial phase}
        \label{fig:train_test_initial}
    \end{subfigure}
    \begin{subfigure}[b]{0.32\textwidth}
        \includegraphics[width=\textwidth]{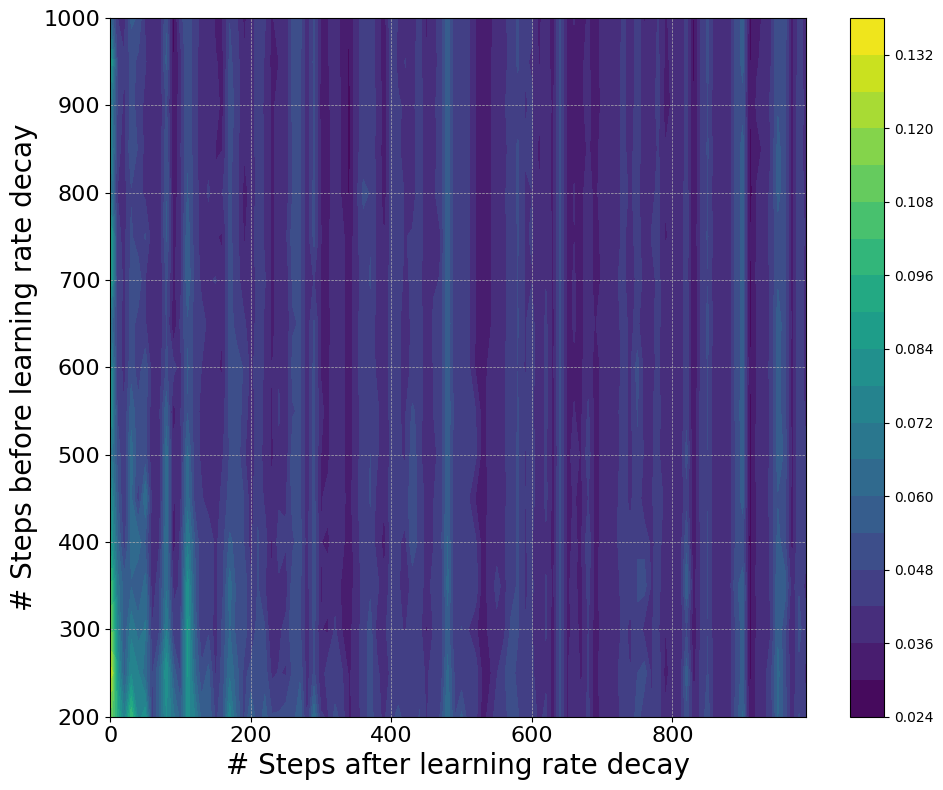}
        \caption{Training loss}
        \label{fig:train_loss}
    \end{subfigure}
    \begin{subfigure}[b]{0.32\textwidth}
        \includegraphics[width=\textwidth]{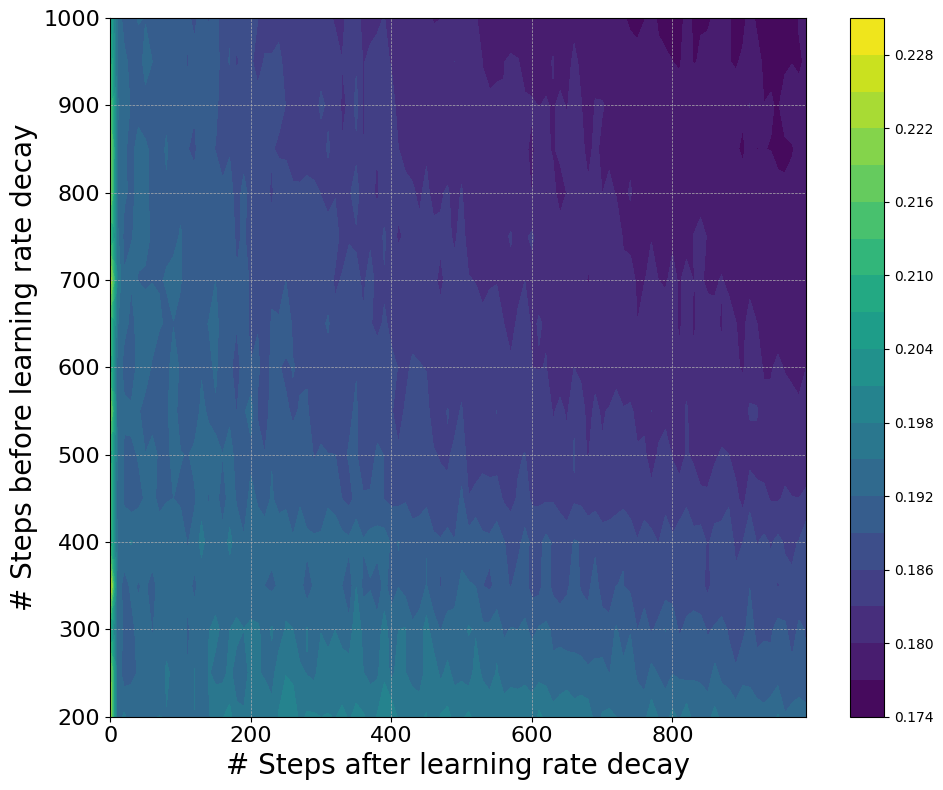}
        \caption{Test loss}
        \label{fig:test_loss}
    \end{subfigure}    
    \caption{Bahaviors of the training and testing losses for a VGG-11 model trained on the CIFAR-10 dataset under various learning rate schedules. Panel (a) showcases the learning curves of the main path with a learning rate of 0.1. In panels (b) and (c), the $y$-axis represents the number of epochs before the learning rate decay, and the $x$-axis indicates the number of epochs after the decay. Each slice parallel to the $x$-axis illustrates the learning curve of a subpath originating from the same main path as shown in (a) with a learning rate of 0.01.}
    \label{fig:learning_curve}
\end{figure*}

We focus on the image classification task on the CIFAR-10 dataset, one of the widely adopted and studied examples in the theory of neural networks~\citep{smith2017don,ma2022beyond,jelassi2022towards}, and recently also in the literature of the Edge of Stability (EoS)~\citep{cohen2021gradient,arora2022understanding}. We employ a VGG-11 model~\citep{simonyan2014very}, trained with the SGD optimizer. The training spans 1000 epochs with a batch size of 128 and an initial learning rate of 0.1, to which we will refer as the \emph{main path}. Starting from epoch 200 of the main path, we introduce additional \emph{subpaths} of other 1000 epochs, during which the learning rate is reduced to 0.01. Due to the large number of subpaths, we only plot the learning curves of the main path (shown in Figure~\ref{fig:train_test_initial}) and present the behaviors of the training and testing losses of subpaths in the form of contour maps in Figure~\ref{fig:train_loss} and~\ref{fig:test_loss}.

As depicted in Figure~\ref{fig:learning_curve}, the timing of the learning rate decay has a small impact on the subsequent training phase (slices parallel to the $x$-axis) with the reduced learning rate. This becomes particularly clear when the training loss stabilizes after the initial $\sim$500 epochs, as evidenced by the seemingly chaotic pattern in Figure~\ref{fig:train_loss}. In contrast, the testing loss reveals a distinct correlation with the timing of learning rate decay. As illustrated in Figure~\ref{fig:test_loss}, delaying the reduction in learning rate results in a lower final testing loss. This observation is consistent with the findings of~\citet{wu2020direction,li2019towards,andriushchenko2022sgd} that training with a larger learning rate enhances generalization.

\subsection{Visualization of Training and Testing Landscape}

\begin{figure*}[!htb]
    \centering
    \begin{subfigure}[b]{0.45\textwidth}
        \includegraphics[width=\textwidth]{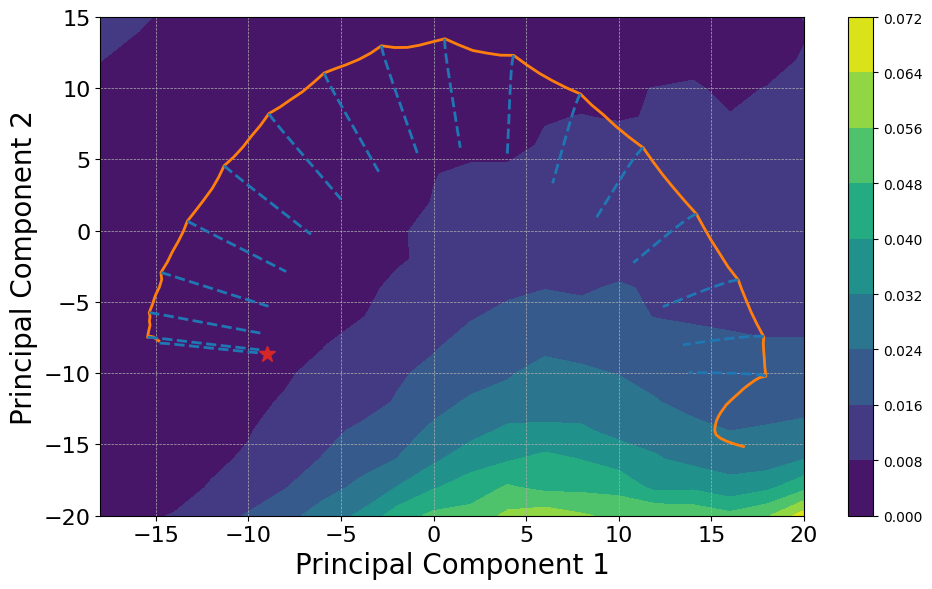}
        \caption{Training loss landscape}
        \label{fig:train_landscape}
    \end{subfigure}
    \begin{subfigure}[b]{0.45\textwidth}
        \includegraphics[width=\textwidth]{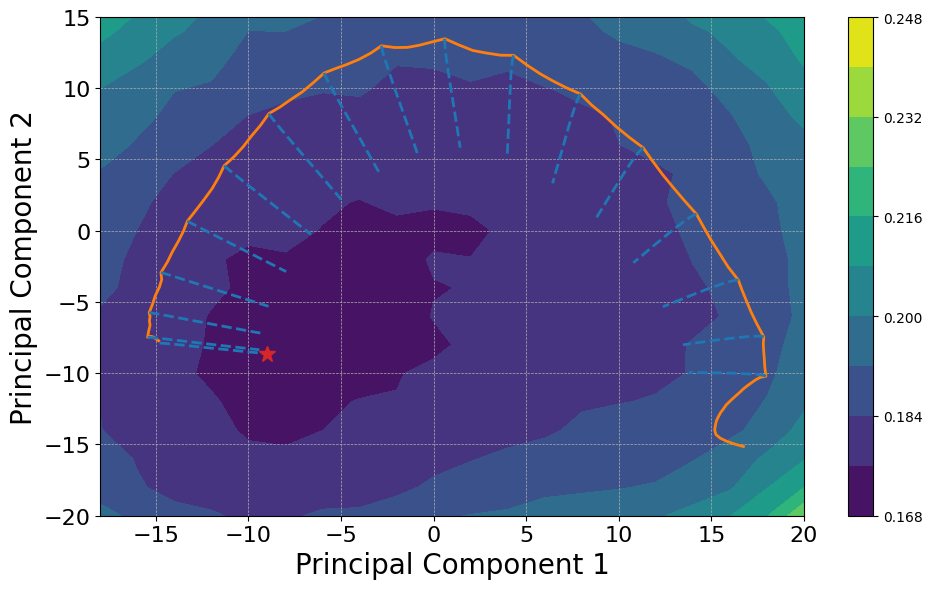}
        \caption{Test loss landscape}
        \label{fig:test_landscape}
    \end{subfigure}    
    \caption{Visualization of the training and testing loss landscapes for a VGG-11 model trained on the CIFAR-10 dataset. The main path with the initial learning rate of 0.1 is represented by an {\color{orange} orange} line, while the subpaths with the reduced learning rate of 0.01 are depicted in {\color{cyan} blue} dashed lines. The final point of the most extended training trajectory that spans 2000 epochs is marked with a red star.}
    \label{fig:loss_landscape}
\end{figure*}

In order to analyze this phenomenon, we proceed by examining the loss landscapes of the neural networks. As illustrated in Figure~\ref{fig:loss_landscape}, we visualize the training and testing loss landscapes for the model. 
We derive these loss landscapes based on the previous works of~\citet{lorch2016visualizing,li2018visualizing}: (a) Collating and flattening the model's parameters at each epoch's conclusion for both the main path and subpaths, (b) Applying Principal Component Analysis (PCA) on these parameters to extract the two primary components, and (c) Calculating the training and test losses over a grid defined by these two principal components, centered around the mean when performing the PCA. 

The primary two principal components account for 77.93\% of the total variances, with each subsequent component explaining less than 5\%. It is important to note that given the high dimensionality of neural networks, the full parameter space might exhibit more complex features not captured in our 2D representation. Furthermore, our visualization only aims to present the local landscape around the minimizers rather than a global one as in~\citet{li2018visualizing}. Nevertheless, our findings still offer valuable insights into the structure of the loss landscapes and can be instrumental in explaining practical training behaviors, as we discuss subsequently.

\subsection{Intuitive Cause of the Landscape Structure}

As showcased in Figure~\ref{fig:train_landscape}, the training loss landscape reveals a low-dimensional manifold of minima (with near-zero training loss). This coincides with the prevalent belief that neural networks, being overparameterized relative to the training sample count, manifest a zero-loss manifold that all training trajectories converge to and then oscillate around~\citep{cooper2018loss, cooper2020critical, li2021happens}. Recent studies~\citep{wu2020direction,li2019towards,andriushchenko2022sgd} demonstrate that the implicit bias of SGD with an initially large learning rate offers generalization advantages. This observation is exactly captured in Figure~\ref{fig:test_landscape}. Contrasting the training loss landscape, the test loss landscape showcases an isolated minimum, with closed loss level sets encircling the minimum. While the training trajectory with a larger learning rate (denoted in orange) traverses the minima manifold of the training loss, its correlation with test loss reduction is tenuous. However, this traversal does lead to a lower final testing loss when the learning rate decays later (indicated in blue).

Based on these empirical findings, we hypothesize that despite the misalignment of the testing loss minimum and the minima manifold of the training loss, longer training with larger learning rates helps find the testing loss minimum. Then, with the decayed learning rate, the trajectory achieves better final generalization. The argument in~\citet{wu2018sgd,mulayoff2021implicit,nacson2022implicit,andriushchenko2023we} for this phenomenon is that SGD with a larger learning rate identifies flatter minima of the training loss (characterized by the Hessian matrix), which are believed to generalize better~\citep{keskar2016large,zhou2020towards}. However, this argument is not sufficient to explain the observed discrepancy between the training and testing loss landscapes.

\section{AN ILLUSTRATIVE MODEL}

In this section, we consider a minimal nonlinear model that helps provide intuition for neural networks with more complicated architectures in the aforementioned phenomenon. We argue that the observed loss landscapes are caused by the nonlinearity of neural networks resulting from the composition of the layers.

\subsection{Model Settings}

Let $\vx\in\sR^d$ be the feature vector, $y\in\sR$ be the label, and $\vw\in\sR^d$ be the model parameter. Our model is defined as 
\begin{equation}
    y = \|\vw\|^\gamma\vw^\top\vx := \valpha^\top \vx,
    \label{eq:model}
\end{equation}
where $\|\cdot\|$ denotes the $L^2$-norm, $\gamma\geq 0$ is a parameter reflecting the depth of the neural network and controlling the shrinkage, which we will explain afterward. $\valpha = \|\vw\|^\gamma \vw$ denotes the effective linear predictor when interpreting~\Eqref{eq:model} as a reparametrization of the linear regression model. It is worth noting that the mapping between $\vw$ and $\valpha$ is invertible with $\|\valpha\|^{-\frac{\gamma}{1+\gamma}}\valpha =\vw$. Thus, when no confusion arises, this bijection is always implicitly assumed, \emph{i.e.} $\valpha = \valpha(\vw)$ and $\vw = \vw(\valpha)$. We adopt the quadratic loss defined as 
\begin{equation*}
    \ell(\vx,y;\vw) = \dfrac{1}{2}(\|\vw\|^\gamma\vw^\top\vx-y)^2.
\end{equation*}

Suppose we have a dataset $\gD = \{(\vx_i,y_i)\}_{i=1}^n$ of $n$ samples, constructed with a ground truth weight $\vw^*$, \emph{i.e.} $y_i = \|\vw^*\|^\gamma(\vw^*)^\top \vx_i := (\valpha^*)^\top \vx_i$ for $i\in[n]$. Define the data matrix $\mX = (\vx_1,\cdots,\vx_n)$, the training loss (or the empirical loss) can be expressed as $$\hat{\gL}(\vw;\vw^*,\gD) = \frac{1}{n}\sum_{i=1}^n \ell(\vx_i,y_i;\vw) = \dfrac{1}{2n}\left\|\mX^\top (\valpha - \valpha^*) \right\|^2.$$ 

We further assume the data are drawn from a standard Gaussian distribution, \emph{i.e.} $\vx_i \sim \gN(0, \mI)$, for $i\in[n]$, the testing loss (or the population loss) is thus
$$
\begin{aligned}
    &\gL(\vw;\vw^*) = \E_{\vx_i\sim \gN(0, \mI)}[\hat{\gL}(\vw;\vw^*;\gD)]\\
    = &\dfrac{1}{2n}\E_{\mX}\left[(\valpha - \valpha^*)^\top \mX\mX^\top (\valpha - \valpha^*)\right] = \dfrac{1}{2}\|\valpha - \valpha^*\|^2.
\end{aligned}
$$ 
Now that we consider the over-parametrized regime, we assume $d > n$ and $\mX$ has full column rank. We will also denote the space spanned by the columns of $\mX$ by $\mX$ and the orthogonal complement of $\mX$ by an orthogonal matrix $\mX^\perp$. For any vector $\vx$, we denote its normalized version as $\overline{\vx} = \vx/\|\vx\|$. For any matrix $\mY$, define the projection operator $\gP_{\mY} = \mY (\mY^\top \mY)^{-1} \mY^\top$ as the projection onto the column space of $\mY$. We define the set of global minima of the empirical loss, which forms the \emph{minima manifold} $\gM$ in the parameter space as 
\begin{equation}
    \gM = \left\{ \vw \big| \valpha - \valpha^* = \|\vw\|^\gamma \vw - \|\vw^*\|^\gamma \vw^* \in \mX^\perp \right\}.
    \label{eq:gM}
\end{equation}

\subsection{Motivation}

\begin{figure*}[!tb]
    \centering
    \begin{minipage}{0.32\textwidth}
        \begin{subfigure}[b]{\textwidth}
            \includegraphics[width=\textwidth]{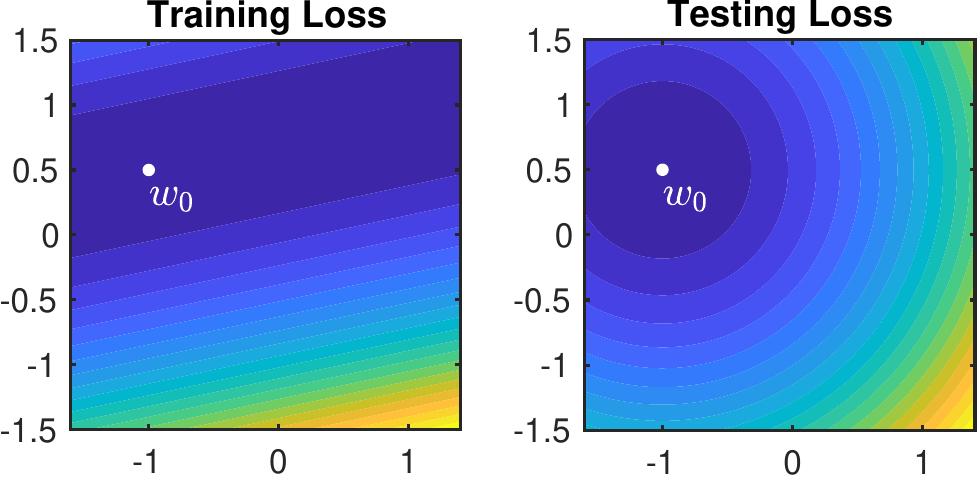}
            \caption{Linear regression model}
            \label{fig:linear}
        \end{subfigure}
    \end{minipage}
    \begin{minipage}{0.32\textwidth}
        \begin{subfigure}[b]{\textwidth}
            \includegraphics[width=\textwidth]{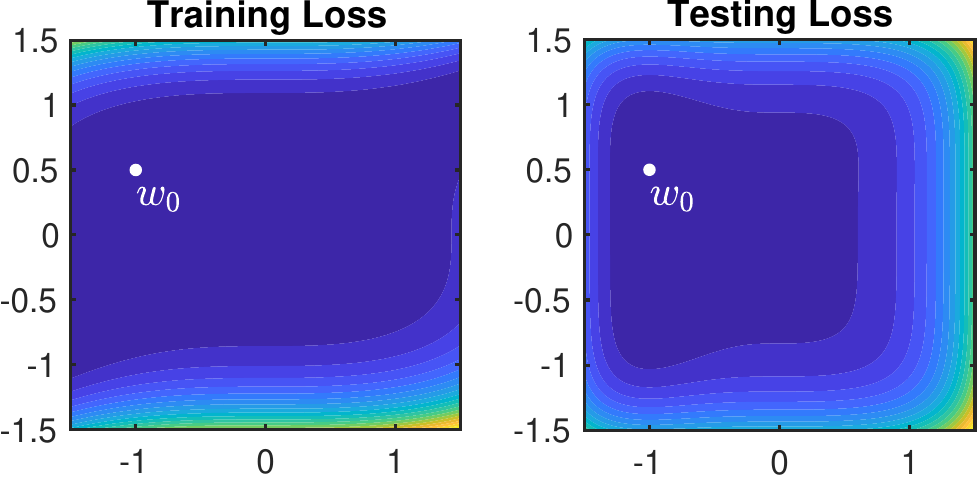}
            \caption{Linear diagonal network}
            \label{fig:diagonal}
        \end{subfigure}
    \end{minipage}
    \begin{minipage}{0.32\textwidth}
        \begin{subfigure}[b]{\textwidth}
            \includegraphics[width=\textwidth]{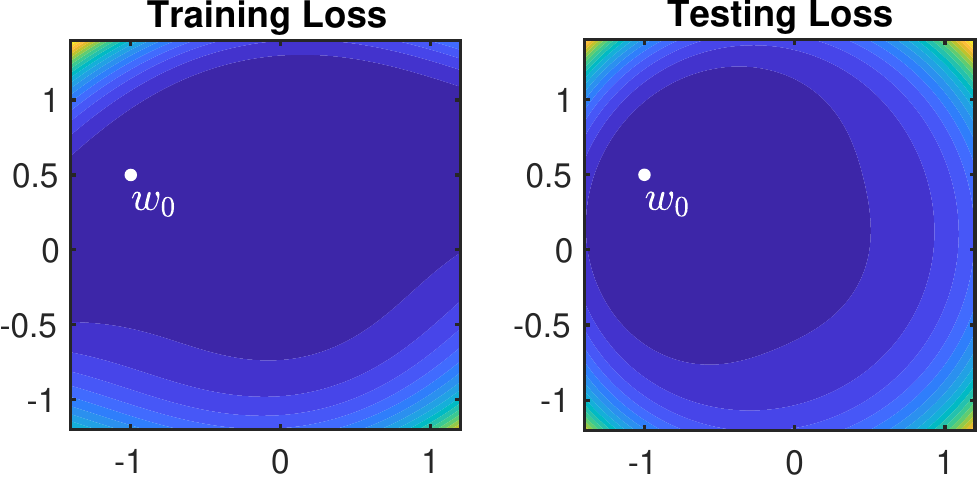}
            \caption{Our reparametrization model}
            \label{fig:norm}
        \end{subfigure}
    \end{minipage}
    \caption{Comparison of training and testing loss landscapes of the linear regression model $y = \vw^\top \vx$~\citep{wu2020direction}, the linear diagonal network model $y = (\vw^{\odot 3})^\top \vx$~\citep{gunasekar2018implicit}, and our reparametrization model $y = \|\vw\|^2 \vw^\top \vx$. In this example, we choose $d=2$, $n=1$, $\vw^* = (-1, 0.5)^\top$ and $\mX = (0.15, -0.7)^\top$.}
\end{figure*}

One of the most studied models in the related literature is the \emph{linear diagonal network} first proposed by~\citet{gunasekar2018implicit}, \emph{i.e.} the reparametrization scheme
\begin{equation}
    \valpha' = \diag(\vw_L) \diag (\vw_{L-1})\cdots \diag(\vw_2) \vw_1.
    \label{eq:linear_diagonal}
\end{equation} 
Under this reparametrization, we denote the empirical and population losses by $\hat{\gL}'(\vw;\vw^*,\gD)$ and $\gL'(\vw;\vw^*)$ respectively, where $\vw = (\vw_1, \cdots, \vw_L)$.
The simplified version of this model $\valpha' = \vw ^{\odot L}$, where $\odot$ denotes entry-wise multiplication, along with its many variations, have been analyzed in~\citet{vaskevicius2019implicit,woodworth2020kernel,pesme2021implicit,haochen2021shape}. This model, especially the two-layer version, has been widely adopted as the toy model for its tractability when taking gradients and allowing for explicit and even discrete time analysis~\citep{haochen2021shape}. However, as we show in Figure~\ref{fig:diagonal}, this model has peculiar training and testing landscapes, which may not capture the real features of the loss landscapes of practical neural networks in Figure~\ref{fig:loss_landscape}. 

We consider adding $L^2$ regularization beside the linear diagonal network in~\Eqref{eq:linear_diagonal} and minimizing the total loss
\begin{equation*}
    \hat\gL'(\vw; \vw^*, \gD) + \lambda \sum_{i=1}^L\|\vw_i\|^2.
\end{equation*}
For each $i\geq 2$, fixing the rest of the parameters, $\vw_i$ is encouraged to have identical entries by the regularization term. This does not affect the expressiveness of the model as long as the first layer $\vw_1$ is not restricted. Moreover, the regularization also encourages $\vw_i$, $i\in[n]$ to have identical $L^2$-norms and the linear diagonal network is thus reduced to our model in~\Eqref{eq:model}, with $\gamma = L-1$.

The additional weight $\|\vw\|^\gamma$ in our model intuitively warps the parameter space from that of the linear regression model (Figure~\ref{fig:linear}) so that the loss landscape is no longer quadratic. While the loss landscapes of our model (Figure~\ref{fig:norm}) preserve the key characteristics of the linear regression model, such as almost quadratic testing loss with an isolated minimum, the level sets of the training loss exhibit a distinctive ``shrinkage'' as $\|\vw\|\to\infty$. This is attributed to the nonlinearity introduced by the depth in neural networks. 

Compared with the linear diagonal network (Figure~\ref{fig:diagonal}), the implicit regularization effect of our model appears more ``isotropic''. As a result, the loss landscape of our model is closer to that of the practical neural networks (Figure~\ref{fig:loss_landscape}). To the best of our knowledge, our model has not been studied by the related literature in neural network theory.

\subsection{Main Results}

In the following, we will analyze the training process of our model using SGD in detail. 
We first characterize this process into the following three phases and delve into the training behaviors exhibited within each:
\begin{enumerate}[label=\Roman*.]
    \item {\bf Initial phase with a large learning rate $\eta_L$:} In this phase, the actual gradient outweighs the noise in SGD, keeping the trajectory close to that of the gradient flow. The decrease in the training loss comes to saturation, and the trajectory approaches the minima manifold $\gM$.
    \item {\bf Extended phase maintaining the large learning rate $\eta_L$:} In this phase, the trajectory actively navigates through the minima manifold $\gM$ driven by the shape of the training loss landscape, during which the training loss only fluctuates, while the trajectory approaches the neighborhood of the minimum $L^2$-norm solution of the training loss in $\Theta(\eta_L^{-1})$ time.
    \item {\bf Final phase with a small learning rate $\eta_S$:} In this concluding phase, the trajectory realigns with the gradient flow, which rapidly penetrates into the minima manifold $\gM$ and the final testing performance depends on the timing of the learning rate decay from the preceding phase.
\end{enumerate}

\subsubsection{Phase I}
\label{sec:phase_I}

Following previous works~\citep{li2017stochastic,li2019stochastic,li2021validity,mori2022power}, we represnet the initial phase with a large learning rate $\eta_L$ with the following stochastic differential equation (SDE):
\begin{equation}
    \rd \vw(t) = -\nabla_\vw \hat \gL(\vw(t); \vw^*,\gD) \rd t + \sqrt{\eta_L} \rd \mB(t),
\label{eq:1_sde}
\end{equation}
where $\mB(t)$ is the standard Brownian motion. The dynamics $\vw(t)$ reflect the discrete updates of the parameter $\vw$ starting from an initialization denoted by $\vw_0$. The time correspondence can be expressed as $\vw(0)=\vw_0$ and $\vw(k\eta_L)$ approximating the parameter value post $k$ iterations. Within this framework, a large learning rate is symbolized by a large diffusion coefficient $\sqrt{\eta_L}$. 

One should notice that~\Eqref{eq:1_sde} represents the Langevin dynamics featuring a Gibbs-type stationary distribution $\exp(-\hat \gL(\vw; \vw^*,\gD)/\eta_L)$, which is concentrated around the minima manifold $\gM$ of the empirical loss. Intuitively, during Phase I, the parameter $\vw$ remains at a considerable distance from the minima manifold $\gM$. Consequently, the dynamics $\vw(t)$ are primarily driven by the gradient of the training loss until the parameter $\vw$ approaches the minima manifold $\gM$, and the noise within the stochastic gradient surpasses the gradient itself.

Instead of the hitting time analysis in the literature of stochastic gradient Langevin dynamics~\citep{zhang2017hitting,chen2020stationary}, which relies on the assumption of the upper bound of $\tr\nabla^2 \hat \gL(\vw; \vw^*,\gD)$, we quantify the time span of Phase I through a mixing time analysis:
\begin{theorem}
    For any initilization $\vw(0)=\vw_0$, under the dynamics in~\Eqref{eq:1_sde}, the dynamics of the projection of $\vw(t)$ onto the column space of $\mX$, denoted by $\vw_\mX(t)$, have \emph{exponential mixing property}, \emph{i.e.} for any two initializations $\vw_0$ and $\vw_0'$, after time $\gO(\log \delta^{-1})$, we have 
    \begin{equation}
        \|P_0^t(\vw_0, \cdot) - P_0^t(\vw_0', \cdot)\|_{TV} \leq \delta,
    \end{equation}
    where $P_0^t(\vw_0, \cdot)$ is the distribution of $\vw_\mX(t)$ starting from $\vw_0$ at time $t$.
\label{thm:phase_I}
\end{theorem}

Intuitively, $\delta$ controls the stability of the dynamics $\vw_\mX(t)$ in the column space of $\mX$~\citep{del2018exponential}, and thus that of the dynamics $\vw(t)$. After a mixing time of $\gO(\log \delta^{-1})$, $\vw(t)$ forgets its initialization and is redistributed according to the Gibbs energy of the empirical loss $\exp(-\hat \gL(\vw; \vw^*,\gD)/\eta_L)$. As a result, the dynamics $\vw(t)$ are close to the minima manifold $\gM$ of the training loss towards the end of Phase I.

The proof of Theorem~\ref{thm:phase_I} is technical and deferred to Appendix~\ref{app:phase_I}. The main idea is to eliminate the dynamics of $\vw_{\mX^\perp}(t)$ by focusing on the SDE of $\vw_\mX(t)$:
\begin{equation}
    \rd \vw_\mX(t) = \vb(t, \vw_\mX(t)) \rd t + \sqrt{\eta_L} \rd \mB_\mX(t),
\label{eq:1_sde_proj}
\end{equation}
where the time-inhomogeneous drift coefficient $\vb(t, \vw_\mX(t))$ is given by
$$
    \vb(t, \vw_\mX(t)) = - \gP_{\mX} \nabla_\vw \hat \gL(\vw(t); \vw^*,\gD).
$$
By choosing a Lyapunov function of the form $\exp(\alpha \|\vw_\mX\|)$, we are able to verify the Foster-Lyapunov condition~\citep{meyn1993stability,kulik2017ergodic} and the local Dobrushin contraction condition~\citep{del2017stochastic}, which are sufficient for the exponential mixing property of the dynamics $\vw_\mX(t)$.

\begin{remark}
    We focus exclusively on the dynamics of $\vw_\mX(t)$ within the column space of $\mX$ due to the following considerations: The level sets of the empirical loss $\hat \gL(\vw; \vw^*,\gD)$ are not compact. Consequently, the dynamics associated with the projection of $\vw(t)$ onto $\mX^\perp$, represented as $\vw_{\mX^\perp}(t)$, might not exhibit mixing properties. However, in Phase I, our primary anticipation is the stabilization of $\vw_\mX(t)$. This ensures that $\vw(t)$ closely aligns with the minima manifold $\gM$ of the training loss. A close examination of the dynamics of $\vw_{\mX^\perp}(t)$ is reserved for our analysis in Phase II.
\end{remark}

\subsubsection{Phase II}
\label{sec:phase_II}

During Phase II, the parameter trajectory $\vw(t)$ already resides near the minima manifold $\gM$ of the training loss and will be confined to a neighborhood of $\gM$ (as depicted in the region between two dashed lines in Figure~\ref{fig:model}). Due to the large learning rate, rather than converging directly towards the minima manifold $\gM$, the trajectory of $\vw(t)$ will meander across the manifold, steered by the flatness of the training loss landscape.

\begin{figure}[!htb]
    \centering
    \includegraphics[width=.9\linewidth]{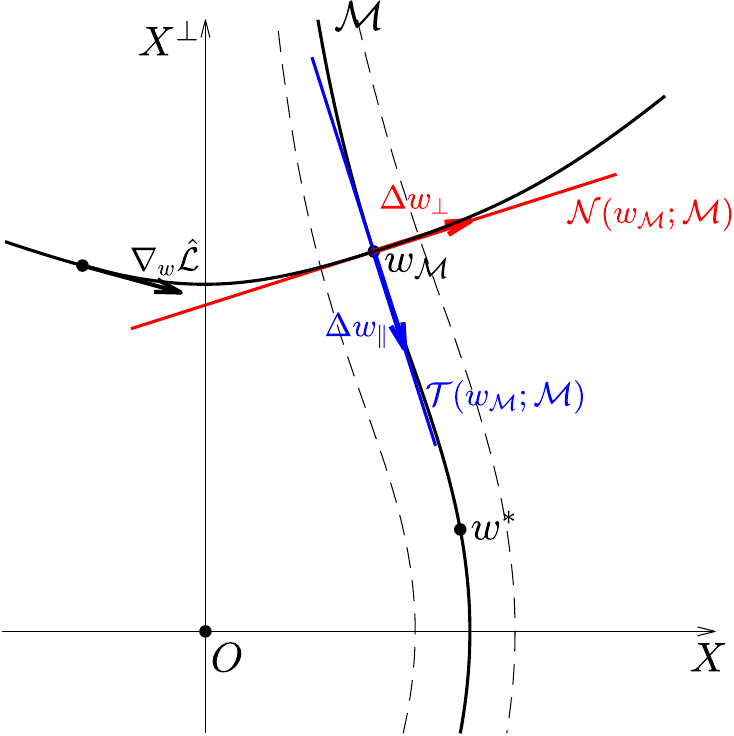}
    \caption{Illustration of the normal space $\gN(\vw_\gM;\gM)$ (highlighteed in {\color{red} red}) and the tangent space $\gT(\vw_\gM;\gM)$ (highlighted in {\color{blue} blue}) of the manifold $\gM$ around a point $\vw_\gM$ on $\gM$. The gradient flow trajectory starting from $\vw(0)$ is represented by a black curve. Due to the stochastic gradient, the dynamics of $\vw(t)$ do not exactly follow the gradient flow but still enter the neighborhood of $\gM$, depicted as the region between the two dashed lines in Phase I. During Phase II, we focus on the effective dynamics of $\vw(t)$ along the minima manifold $\gM$, denoted by $\vw_\gM(t)$. The time scale separation of the dynamics in the normal and tangent spaces allows a quasistatic approach for the analysis of $\vw_\gM(t)$.}
    \label{fig:model}
\end{figure}

Grounded in the insights from~\citet{li2021happens,ma2022beyond}, we sidestep the landscape of the empirical loss $\hat \gL(\vw; \vw^*,\gD)$ where $\vw$ is far from $\gM$ and only consider the \emph{effective dynamics} of $\vw(t)$ along the minima manifold $\gM$, denoted by $\vw_\gM(t)$. Specifically, we consider the quadratic expansion of $\hat \gL(\vw; \vw^*,\gD)$ around the minima manifold $\gM$ of the training loss, and
analyze the dynamics along the manifold by a \emph{quasistatic} analysis: (a) Projecting the dynamics of $\vw(t)$ onto the tangent space and the normal space of the manifold $\gM$, (b) Determining the stationary distribution of the fast projected dynamics $\Delta\vw_\perp(t)$ in the normal space, and (c) Identifying the effective dynamics in the tangent space by averaging the projected dynamics in the normal space.

\paragraph{SGD with Label Noise.} In Phase II, we consider an SDE of the general form:
\begin{equation}
    \rd \vw(t) = - \nabla_\vw \hat{\gL}(\vw(t); \vw^*, \gD) \rd t + \sqrt{\eta_L} \msigma(\vw(t)) \rd \mB(t),
\label{eq:2_sde}
\end{equation}
where $\msigma(\cdot)\in \R^{d\times d}$ is the diffusion tensor, which we will specify later. Fix the time $t$ and a point $\vw_\gM$ on $\gM$ near $\vw(t)$, we decompose $\vw(t) - \vw_\gM$ into the tangent component $\Delta \vw_\parallel(t)$ along the tangent space $\gT(\vw_\gM;\gM)$ and the normal component $\Delta\vw_\perp(t)$ with respect to the normal space $\gN(\vw_\gM;\gM)$, \emph{i.e.} 
\begin{equation}
    \vw(t) - \vw_\gM = \Delta \vw_\parallel(t) + \Delta\vw_\perp(t).
\label{eq:2_decomposition}
\end{equation} 
By Proposition~\ref{prop:space} and $\mA_\gM$ defined therein, we are able to perform the projections with closed forms and rewrite~\Eqref{eq:2_sde} as the following system of SDE:
\begin{equation}
    \begin{cases}
        \begin{aligned}
            \rd \Delta \vw_\parallel(t) = - \gP_{\mA_\gM^{-1}\mX^\perp} \nabla_\vw& \hat{\gL}(\vw(t); \vw^*, \gD) \rd t\\
            & +\sqrt{\eta_L} \msigma_{\parallel}(\vw(t)) \rd \mB(t),
        \end{aligned}\\
        \begin{aligned}
            \rd \Delta \vw_\perp(t) = - \gP_{\mA_\gM\mX} \nabla_\vw \hat{\gL}&(\vw(t); \vw^*, \gD) \rd t \\ & + \sqrt{\eta_L} \msigma_{\perp}(\vw(t)) \rd \mB(t),
        \end{aligned}
    \end{cases}
    \label{eq:2_sde_proj}
\end{equation}
where $\msigma_\parallel(\cdot) = \gP_{\mA^{-1}_\gM \mX^\perp} \msigma(\cdot)$ and $\msigma_\perp(\cdot) = \gP_{\mA_\gM \mX} \msigma(\cdot)$ are the diffusion tensors in the tangent and normal space respectively.

As studied in~\citet{blanc2020implicit,damian2021label}, the scenario where the labels $\{y_i\}_{i=1}^n$ of the dataset $\gD$ contains binary noises, \emph{i.e.} $y_i = (\valpha^*)^\top \vx_i + \xi_i$ where $\xi_i \sim \unif\{-\sigma, \sigma\}$, corresponds to the diffusion tensor $\msigma$ satisfying the following assumption:
\begin{equation}
    \msigma(\vw(t)) \msigma(\vw(t))^\top := \sigma^2 \nabla^2 \hat{\gL}(\vw(t); \vw^*, \gD),
    \label{eq:2_assumption}
\end{equation}
causing $\msigma_{\parallel}(\vw(t))$ to vanish, for $\vw(t)$ near $\gM$. 

\paragraph{Stationary Distribution along the Normal Space.} 
Expand the empirical loss $\hat{\gL}(\vw; \vw^*, \gD)$ around $\vw_\gM$ (\emph{cf.}~\Eqref{eq:L_expansion}), we have
\begin{equation}
    \nabla_\vw \hat{\gL}(\vw_\gM +\Delta \vw; \vw^*, \gD) = \dfrac{1}{n} \mA_\gM \mX\mX^\top (\valpha - \valpha^*)+ o(\|\Delta\vw\|),
\label{eq:2_grad}
\end{equation}
from which we see that the first order term of $\hat{\gL}(\vw; \vw^*, \gD)$ is in the normal space and thus the drift term of the dynamics of $\Delta\vw_\perp(t)$ is of the first order of $\Delta \vw$, while the drift term of the dynamics of $\Delta\vw_\parallel(t)$ is at most of the second order of $\Delta \vw$. This indicates that the dynamics of $\Delta \vw_\parallel(t)$ are much slower than that of $\Delta \vw_\perp(t)$. Therefore, we consider rescaling the dynamics of $\Delta \vw_\perp(t)$ by $\Delta \vw_\perp(t+\tau)$ from time $t$ with a smaller time scale $\tau$. In fact, as we will see in Lemma~\ref{lem:effective_dynamics}, the scale of $t$ is $\Theta(\eta_L^{-1})$ greater than that of $\tau$. 

By choosing the the parametrization $\Delta \vw_\perp(t+\tau) = \mA_\gM \mX \veps(\tau)$, where 
$\veps(\tau)\in\R^n$, and expanding around $\vw_\gM$, one can show that the rescaled dynamics of $\Delta \vw_\perp(t+\tau)$ under the assumption~\eqref{eq:2_assumption} is simplified to the following Ornstein-Uhlenbeck (OU) process
\begin{equation}
    \rd \veps(\tau) = - \dfrac{1}{n}\mX^\top \mA_\gM^2\mX \veps(\tau)\rd \tau  + \sqrt{\dfrac{\sigma^2\eta_L}{n}} \rd \mB_\perp(\tau),
\label{eq:2_ou}
\end{equation}
where $\mB_\perp$ denotes an $n$-dimensional Brownian motion. Notice that the above OU process has exponentially mixing property, whereby we approximate the dynamics of $\Delta\vw_\perp(t)$ at the large time scale as
\begin{equation}
    \Delta\vw_\perp(t) \sim \gN\left(\vzero, \dfrac{\sigma^2\eta_L}{2}  \mA_\gM \mX(\mX^\top \mA_\gM^2 \mX)^{-1}\mX^\top \mA_\gM\right),
\label{eq:2_stationary}
\end{equation}
\emph{i.e.} assuming $\Delta\vw_\perp(t)$ is at the stationary distribution of the OU process~\Eqref{eq:2_ou} at time $t$.

\paragraph{Effective Dynamics along the Tangent Space.} We then consider the effective dynamics of $\Delta \vw_\parallel(t)$ along the tangent space $\gT(\vw_\gM;\gM)$ by taking expectation of the stationary distribution~\eqref{eq:2_stationary} of the rescaled dynamics $\Delta \vw_\perp(t+\tau)$ in the normal space $\gN(\vw_\gM;\gM)$ for each time $t$, \emph{i.e.}
\begin{equation}
    \dfrac{\rd \Delta \vw_\parallel(t)}{\rd t} = - \E_{\Delta \vw_\perp(t)}  \left[\gP_{\mA_\gM^{-1}\mX^\perp} \nabla_\vw \hat{\gL}(\vw(t); \vw^*, \gD)\right].
\label{eq:2_effective}
\end{equation}

Summarizing the above analysis, we have the following characterization of the dynamics of $\vw(t)$ in Phase II:
\begin{lemma}
    Under the dynamics in~\Eqref{eq:2_sde_proj} with the label noise assumption~\citep{blanc2020implicit,damian2021label}, the effective dynamics of $\vw(t)$ in Phase II are given by
    \begin{equation}
         \dfrac{\rd \vw_\gM(t)}{\rd t} = - A\|\vw_\gM(t)\|^{2\gamma-1} C(\vw_\gM(t))\gP_{ \mA_\gM^{-1}\mX^\perp}\overline{\vw_\gM(t)},
    \label{eq:2_effective_dynamics}
    \end{equation}
    where $A = \eta_L\sigma^2 \gamma/2$ and
    \begin{equation}
        C(\vw_\gM) =\dfrac{1}{n} \left(\tr(\mX^\top \mX) - (\gamma+2)\overline{\vw_\gM}^\top \mX \mX^\top \overline{\vw_\gM}\right).
    \label{eq:2_C}
    \end{equation}
    \label{lem:effective_dynamics}
\end{lemma}
\vspace{-2em}
The $\eta_L$ factor in $A$ reflects the time scale separation between the dynamics of $\vw(t)$ and $\vw_\gM(t)$ as we claimed before. The effective dynamics~\eqref{eq:2_effective_dynamics} have a clear geometric interpretation: as shown in Figure~\ref{fig:model}, whenever $C(\vw_\gM(t))\geq 0$, we have $\dot \vw_\gM\propto -\gP_{ \mA_\gM^{-1}\mX^\perp}\overline{\vw_\gM(t)}$ pointing towards the column space of $\mX$, which leads to the following result: 
\begin{theorem}
    For any $\gamma \geq 0$, the effective dynamics~\eqref{eq:2_effective_dynamics} converge to the minimum $L^2$-norm solution of the population loss $\gL(\vw; \vw^*, \gD)$, denoted by $\vw^\dagger$, with probability $1 - \exp(-\Omega(nd))$. The norm of $\|\vw_\gM(t)\|$ satisfies:
    \begin{equation}
        \dfrac{\rd \|\vw_\gM(t)\|}{\rd t} \leq -B\left(\|\vw_\gM(t)\|^{2\gamma-1}  - \|\vw_\gM(t)\|^{-3} \|\valpha_\mX^*\|^2\right),
    \label{eq:2_ode_3}
    \end{equation}
    where $B = \frac{\sigma^2 \eta_L\gamma\underline{C}d}{2(1+\gamma)^2}$, $\underline{C} = \inf_{t\geq 0} C(\vw_\gM(t))$, and $\valpha_\mX^* = \gP_\mX \valpha^*$.
    Moreover, for any $\gamma > 1/2$, the convergence in~\Eqref{eq:2_ode_3} is exponentially fast, \emph{i.e.} after time 
    $
    \gO\left(\frac{\log \delta^{-1}}{\sigma^2 \eta_L \underline{C}d} \right)
    $,
    we have $\|\vw_\gM(t)\| - \|\vw^\dagger\| \leq \delta$.
\label{thm:phase_II}
\end{theorem}

\begin{remark}
    \label{rem:phase_II}
    We would like to make the following remarks regarding the above theorem:
    \vspace{-.5em}
    \begin{itemize}[leftmargin=*]
        \item The high probability argument is due to the randomness in the data generation of $\mX$. One should notice that $C(\vw_\gM)\geq 0$ holds if the numerical rank of the data matrix $r(\mX) := \|\mX\|_F/\|\mX\|\geq \sqrt{\gamma + 2}$. In high-dimensional spaces ($d\to\infty$), $r(\mX)\gtrsim \sqrt{n}$, and thus $C(\vw_\gM)$ can be lower bounded with high probability (\emph{cf.} Lemma~\ref{lem:positive_C}). 
        \item  Mirroring the $L^2$-regularization or the ridge regression helps avoid overfitting in practice, the minimum $L^2$-norm solution $\vw^\dagger$ of overparametrized models is both intuitively and provably the near-optimal solution under certain assumptions~\citep{wu2020optimal,bartlett2020benign}. Notably, our result here aligns with those in the literature of overparametrized ridgeless regrssion~\citep{liang2020just,hastie2022surprises}.
        \item As observed from~\Eqref{eq:2_ode_3}, the condition $\gamma\geq0$ can be relaxed to $\gamma>-1$, by which we still have $\|\vw_\gM(t)\|^{2\gamma-1} \geq \|\vw_\gM(t)\|^{-3} \|\valpha_\mX^*\|^2$ and thus the convergence of the effective dynamics. This makes our model less restrictive compared with linear diagonal networks~\citep{gunasekar2018implicit,woodworth2020kernel}. We would also like to point out that $\gamma = -1$ refers to the weight normalization~\citep{wu2020implicit,chou2023robust}, where the convergence to the minimum $L^2$-norm solution is also obtained. Still, an additional parameter is needed to control the length of the parameter $\valpha$.
    \end{itemize}
\end{remark}

Further explanations of the above analysis and the proof of the above theorem are deferred to Appendix~\ref{app:phase_II}.

\subsubsection{Phase III}
\label{sec:phase_III}

Suppose the trajectory of $\vw(t)$ reaches the neighborhood of a point $\vw_\gM$ on the minima manifold $\gM$ of the training loss. We then perform a local analysis around $\vw_\gM$. When the step size $\eta_S$ is sufficiently small, the trajectory of performing SGD on the empirical loss $\hat{\gL}(\vw; \vw^*, \gD)$ can be approximated by the gradient flow~\citep{smith2021origin}, \emph{i.e.} 
\begin{equation}
    \rd \vw(t) = -\nabla_\vw \hat{\gL}(\vw(t); \vw^*, \gD)\rd t.
\label{eq:3_sde}
\end{equation}

Again, by approximating the local geometry by the quadratic expansion of $\hat{\gL}(\vw; \vw^*, \gD)$ around $\vw_\gM$, we have the following exponential convergence:
\begin{theorem}
    Under the dynamics in~\Eqref{eq:3_sde}, the parameter $\vw(t)$ converges to a nearby $\vw_\gM\in\gM$ exponentially fast, \emph{i.e.} after time $\gO(\log \delta^{-1})$, we have 
    $\|\vw(t) - \vw_\gM\| \leq \delta$.
\label{thm:phase_III}
\end{theorem}
The proof of this theorem is standard and will be deferred to Appendix~\ref{app:phase_III}.

Recall that by Theorem~\ref{thm:phase_II}, if Phase II is executed over a sufficiently long period, the effective dynamics $\vw_\gM(t)$ converge to the minimum $L^2$-norm solution $\vw^\dagger$, which means the original dynamics $\vw(t)$ approach the neighborhood of $\vw^\dagger$, and then Phase III is able to recover the minimum $L^2$-norm solution $\vw^\dagger$ of the training loss. In general, the final testing performance thus depends on the timing of the learning rate decay from Phase II to Phase III. This corroborates the empirically observed benefits of late learning rate decay.

\vspace{-.5em}
\section{DISCUSSIONS}
\vspace{-.5em}

In this paper, we delve into understanding the question of why late learning rate decay leads to better generalization. Our study is motivated by experimental observations and the visualization of the training and testing loss landscapes on an image classification task. We subsequently introduce an overparametrized model with a novel nonlinear reparametrization, which presents a ``shrinking''  training loss landscape as the $L^2$-norm of the parameter increases. 

Upon establishing this model, we characterize the training process into three distinct phases. Our analysis emphasizes that during Phase II, which corresponds to the extended period before learning rate decay, the parameter approaches the minimum $L^2$-norm solution. We believe our model and results shed light on the training process of neural networks and provide a new perspective on the generalization of deep learning architectures.
One of the limitations of our work is that our analysis is predominantly based on a continuous-time approach, and the discrete-time analysis for our model is left for future work. 

\subsubsection*{Acknowledgments}

We thank the anonymous reviewers for their helpful comments. 

\bibliography{references}

\appendix

\onecolumn

\section{MISSING PROOFS FOR PHASE I}
\label{app:phase_I}

In this section, we present the missing proofs of the results in Phase I in Section~\ref{sec:phase_I} of the main text. 

Before we present the proofs, we first perform some common computations regarding the gradient of the empirical loss $\hat{\gL}(\vw; \vw^*, \gD)$ and the population loss $\gL(\vw; \vw^*)$ for later reference:
\begin{equation}
    \begin{aligned}
        \nabla_\valpha \hat{\gL}(\vw; \vw^*, \gD) &= \dfrac{1}{n} \mX\mX^\top (\valpha - \valpha^*),\\
        \nabla_\vw \valpha &= \|\vw\|^\gamma \mI + \gamma\|\vw\|^{\gamma-1}\dfrac{\vw}{\|\vw\|} \vw^\top = \|\vw\|^\gamma\left(\mI + \gamma \overline \vw \  \overline \vw^\top\right):= \mA(\vw),\\
        \nabla_\vw \hat{\gL}(\vw; \vw^*, \gD) &= \left(\nabla_\vw \valpha\right) \left(\nabla_\valpha \hat{\gL}(\vw; \vw^*, \gD)\right)= \dfrac{1}{n} \mA(\vw) \mX\mX^\top (\valpha - \valpha^*),
    \end{aligned}
\label{eq:gradient}
\end{equation}
where the gradients are assumed to be column vectors and $\overline \vw = \vw/\|\vw\|$ is the normalized parameter.

For readers' convenience, we restate the SDE in~\Eqref{eq:1_sde} here:
\begin{equation*}
    \rd \vw(t) = -\nabla_\vw \hat \gL(\vw(t); \vw^*,\gD) \rd t + \sqrt{\eta_L} \rd \mB(t).
\end{equation*}

Let $\vw_\mX(t)$ be the projection of $\vw(t)$ onto the column space of $\mX$, and $\vw_{\mX^\perp}(t)$ be the projection of $\vw(t)$ onto the orthogonal complement of the column space of $\mX$. Then, the dynamics~\eqref{eq:1_sde} can be decomposed into the following two SDEs:
\begin{equation}
    \begin{cases}
        \rd \vw_\mX(t) &= - \gP_{\mX} \nabla_\vw \hat \gL(\vw(t); \vw^*,\gD) \rd t + \sqrt{\eta_L} \rd \mB_\mX(t),\\
        \rd \vw_{\mX^\perp}(t) &= - \gP_{\mX^\perp} \nabla_\vw \hat \gL(\vw(t); \vw^*,\gD) \rd t + \sqrt{\eta_L} \rd \mB_{\mX^\perp}(t),
    \end{cases}
\label{eq:1_sde_decomp}
\end{equation}
where $\mB_\mX(t)$ and $\mB_{\mX^\perp}(t)$ are the Brownian motions in the column space of $\mX$ and its orthogonal complement $\mX^\perp$, respectively.

One should notice that the two SDEs in~\eqref{eq:1_sde_decomp} are coupled. Now that we are only interested in the dynamics of the projection $\vw_\mX(t)$, we can eliminate the dynamics of $\vw_{\mX^\perp}(t)$ by assuming a more general drift term in the SDE of $\vw_\mX(t)$ as presented in~\Eqref{eq:1_sde_proj}:
\begin{equation*}
    \rd \vw_\mX(t) = \vb(t, \vw_\mX(t)) \rd t + \sqrt{\eta_L} \rd \mB_\mX(t),
\end{equation*}
where the drift coefficient $\vb(t, \vw_\mX(t))$ is a function of both time $t$ and the projection $\vw_\mX(t)$ such that
$$
    \vb(t, \vw_\mX(t)) = - \gP_{\mX} \nabla_\vw \hat \gL(\vw(t); \vw^*,\gD).
$$
We denote the infinitestimal generator of the SDE~\eqref{eq:1_sde_proj} by $\gA_s$, \emph{i.e.}
$$
    \gA_s = \vb(s, \vw_\mX(s))\cdot\nabla + \dfrac{\eta_L}{2} \tr \nabla^2,
$$
where the subscript $s$ indicates the dependence of the infinitesimal generator on time $s$.

In the following, we adapt the Foster-Lyapunov criteria~\citep{meyn1993stability,kulik2017ergodic} to analyze the mixing time of the dynamics of $\vw_\mX(t)$. Should no confusion arise, we will use $\vw_\mX \in \R^n$ to denote the projection of an arbitrary point $\vw\in\R^d$ onto the column space of $\mX$.

Define the Lyapunov function as 
\begin{equation}
    W(\vw_\mX) = \exp(\alpha \| \vw_\mX \|),
\label{eq:lyapunov}
\end{equation}
where $\alpha$ is a positive constant, we have the following lemma:
\begin{lemma}
    For any time $s$, the infinitestimal generator $\gA_s$ of the dynamics of $\vw_\mX(t)$ satisfies the following drift condition:
    \begin{equation*}
        \gA_s W(\vw_\mX) \leq - C_1 W(\vw_\mX) + C_2,
    \end{equation*}
    where $C_1$ and $C_2$ are positive constants.
\label{lem:lyapunov}
\end{lemma}

\begin{proof}
    Since the drift and diffusion tensors $\vb(s, \vw_\mX(s))$ and $\eta_L$ are locally bounded and $V\in C^\infty(\R^n)$, we only have to prove 
    $$
        \limsup_{\|\vw_\mX\|\to\infty} \gA_s W(\vw_\mX) \leq - C_1,
    $$
    for some positive constant $C_1$, which is equivalent to
    $$
        \limsup_{\|\vw_\mX\|\to\infty} \dfrac{\gA_s W(\vw_\mX)}{W(\vw_\mX)} < 0.
    $$

    To compute $\gA_s W(\vw_\mX)$, we first calculate the gradient and the Hessian of $W(\vw_\mX)$:
    $$
        \begin{aligned}
            \nabla W(\vw_\mX) &= \alpha W(\vw_\mX)\overline{\vw_\mX},\\
            \nabla^2 W(\vw_\mX) &= \alpha^2 W(\vw_\mX)\overline{\vw_\mX}\ \overline{\vw_\mX}^\top + \alpha W(\vw_\mX)\dfrac{\mI - \overline{\vw_\mX}\ \overline{\vw_\mX}^\top}{\|\vw_\mX\|},
        \end{aligned}
    $$
    where $\overline{\vw_\mX} = \vw_\mX/\|\vw_\mX\|$. Then we have 
    \begin{equation}
        \begin{aligned}
            \gA_s W(\vw_\mX) &= 
            \alpha W(\vw_\mX)\langle \overline{\vw_\mX}, \vb(s, \vw_\mX) \rangle + \dfrac{\eta_L}{2}  \tr\left(\alpha^2 W(\vw_\mX) \overline{\vw_\mX}\ \overline{\vw_\mX}^\top + \alpha  W(\vw_\mX) \dfrac{\mI - \overline{\vw_\mX}\ \overline{\vw_\mX}^\top}{\|\vw_\mX\|}\right)\\
            &= W(\vw_\mX)\left[ \alpha \langle \overline{\vw_\mX}, \vb(s, \vw_\mX) \rangle  +  \dfrac{\eta_L}{2} \left(\alpha^2 + \alpha \dfrac{n - 1}{\|\vw_\mX\|}\right)  \right],
        \end{aligned}
        \label{eq:lyapunov_1}
    \end{equation}
    where $\langle \cdot, \cdot \rangle$ denotes the inner product.

    Plugging~\eqref{eq:gradient} into~\eqref{eq:lyapunov_1}, we have
    \begin{equation*}
        \begin{aligned}
            \langle \overline{\vw_\mX}, \vb(s, \vw_\mX) \rangle &= -\dfrac{1}{n} \langle \overline{\vw_\mX}, \gP_\mX \nabla_\vw \hat \gL (\vw; \vw^*, \gD)\rangle\\
            &= -\dfrac{1}{n} \langle  \overline{\vw_\mX},  \mA(\vw) \mX\mX^\top (\valpha - \valpha^*)\rangle\\
            &= -\dfrac{\|\vw\|^\gamma}{n} \langle  \overline{\vw_\mX},  \mX\mX^\top(\valpha - \valpha^*)\rangle - \dfrac{\|\vw\|^\gamma}{n} \langle  \overline{\vw_\mX},  \overline \vw \  \overline \vw^\top \mX\mX^\top (\valpha - \valpha^*)\rangle \\
            &= -\dfrac{\|\vw\|^\gamma}{n} \langle  \overline{\vw_\mX},  \mX\mX^\top(\valpha - \valpha^*)\rangle - \dfrac{\|\vw\|^\gamma}{n}\overline{\vw_\mX}^\top \overline \vw \  \overline \vw^\top \mX\mX^\top (\valpha - \valpha^*),
        \end{aligned}
    \end{equation*}
    where the second equality is due to $\overline{\vw_\mX} \in \mX$.

    Notice that $\valpha = \| \vw \|^\gamma \vw$ by definition, we have
    \begin{equation}
        \begin{aligned}
            &\limsup_{\|\vw_\mX\|\to\infty} \dfrac{\langle \overline{\vw_\mX}, \vb(s, \vw_\mX) \rangle}{\|\vw\|^{2\gamma} \|\vw_\mX\|}\\ =&\limsup_{\|\vw_\mX\|\to\infty}  -\dfrac{1}{n} \left\langle  \overline{\vw_\mX},  \mX\mX^\top \left(\overline{\vw_\mX} - \dfrac{\valpha^*}{\|\vw\|^{\gamma} \|\vw_\mX\|}\right) \right\rangle - \dfrac{1}{n}\overline{\vw_\mX}^\top \overline \vw \  \overline \vw^\top \mX\mX^\top \left(\overline{\vw} \dfrac{\|\vw\|}{\|\vw_\mX\|} - \dfrac{\valpha^*}{\|\vw\|^{\gamma}\|\vw_\mX\|}\right)\\
            =& \limsup_{\|\vw_\mX\|\to\infty} -\dfrac{1}{n} \left\langle  \overline{\vw_\mX},  \mX\mX^\top \overline{\vw_\mX}\right\rangle - \dfrac{1}{n}\overline{\vw_\mX}^\top \overline \vw \  \overline \vw^\top \mX\mX^\top \overline{\vw}\dfrac{\|\vw\|}{\|\vw_\mX\|} \\
            \leq& \limsup_{\|\vw_\mX\|\to\infty} -\dfrac{1}{n}  \overline{\vw_\mX}^\top \mX\mX^\top \overline{\vw_\mX} < 0,
        \end{aligned}
    \label{eq:lyapunov_2}
    \end{equation}
    where the first equality is due to $\mX^\top \vw = \mX^\top \vw_\mX$, the second equality is because $\|\vw\|\geq \|\vw_\mX\|$ and hence $\|\vw_\mX\|\to\infty$ implies $\|\vw\|\to\infty$, the second to last inequality is due to $\overline{\vw_\mX}^\top \overline \vw \geq 0$ and $\overline \vw^\top \mX\mX^\top \overline{\vw}$, and the last inequality is due to $\overline{\vw_\mX} \in \mX$ and $\mX$ is full rank. 

    Suppose $\limsup_{\|\vw_\mX\|\to\infty} \frac{\langle \overline{\vw_\mX}, \vb(s, \vw_\mX) \rangle}{\|\vw\|^{2\gamma} \|\vw_\mX\|} \leq -C_3$ for some positive constant $C_3$, we have
    \begin{equation*}
        \begin{aligned}
            \limsup_{\|\vw_\mX\|\to\infty}\dfrac{\gA_s W(\vw_\mX)}{W(\vw_\mX)} &=\limsup_{\|\vw_\mX\|\to\infty} \alpha \langle \overline{\vw_\mX}, \vb(s, \vw_\mX) \rangle  +  \dfrac{\eta_L}{2} \left(\alpha^2 + \alpha \dfrac{n - 1}{\|\vw_\mX\|}\right),\\
            &\leq - \limsup_{\|\vw_\mX\|\to\infty}\alpha C_3 \|\vw\|^{2\gamma}\|\vw_\mX\| + \dfrac{\eta_L}{2} \left(\alpha^2 + \alpha \dfrac{n - 1}{\|\vw_\mX\|}\right)<0,
        \end{aligned}
    \end{equation*}
    which proves our statement.
\end{proof}

In the following, we will use the notion of \emph{Markov transitions}, referring to a transition kernel $T(\cdot, \rd \vw_\mX)$ on $\R^n$. We also define the following two integral operations induced by this transition kernel 
\begin{equation*}
    Tf(\cdot) := \int_{\R^n} f(\vw_\mX) T(\cdot, \rd \vw_\mX) \quad \text{and} \quad \mu T(\cdot) = \int_{\R^n} T(\vw_\mX, \cdot) \rd \mu(\vw_\mX),
\end{equation*}
where $f$ is a bounded measurable function on $\R^n$ and $\mu$ is a probability measure on $\R^n$. As a special family of Markov transitions, we define the time-inhomogeneous Markov transition semigroup $P_{s}^{s+\tau}$ for $\tau\geq0$ as:
\begin{equation*}
    P_s^{s+\tau}f(\vw_\mX) = \E[f(\vw_\mX(t))\mid \vw_\mX(s) = \vw_\mX],
\end{equation*} 
we immediately have the following corollary by Lemma~\ref{lem:lyapunov}:
\begin{corollary}
    For a fixed time interval $\tau\geq 0$, we have for any time $s$ the transition semigroup $P_s^{s+\tau}$ satisfies the following \emph{Foster-Lyapunov} property w.r.t. $W$:
    \begin{equation}
        P_s^{s+\tau}W(\vw_\mX) \leq C_4 W(\vw_\mX) + C_5,
        \label{eq:transition_semigroup}
    \end{equation}
    where $C_4$ and $C_5$ are positive constants.
\label{cor:lyapunov}
\end{corollary}
\begin{proof}
    By Dynkin's formula~\citep{weinan2021applied}, we have 
    \begin{equation*}
        \begin{aligned}
            &\E[e^{C_1\tau}W(\vw_\mX(s+\tau))\mid \vw_\mX(s) = \vw_\mX]- W(\vw_\mX)\\
            =& \E\left[\int_s^{s+\tau} \rd \left(e^{C_1u} W(\vw_\mX(u))\right) \bigg| \vw_\mX(s) = \vw_\mX\right]\\
            =& \E\left[\int_s^{s+\tau} e^{C_1 u} \gA_u W(\vw_\mX(u)) + C_1 e^{C_1 u} W(\vw_\mX(u))  \rd u\bigg| \vw_\mX(s) = \vw_\mX\right]\\
            \leq & \E\left[\int_s^{s+\tau} e^{C_1 u} \left(-C_1 W(\vw_\mX(u)) + C_2\right) + C_1 e^{C_1 u} W(\vw_\mX(u))  \rd u\bigg| \vw_\mX(s) = \vw_\mX\right]\\
            = & \E\left[\int_s^{s+\tau}C_2 e^{C_1 u}  \rd u\bigg| \vw_\mX(s) = \vw_\mX\right]\\
            = & \dfrac{C_2}{C_1}\left(e^{C_1 \tau} - 1\right),
        \end{aligned}
    \end{equation*}
    and thus~\Eqref{eq:transition_semigroup} follows by taking 
    $$
        C_4 = 1 - e^{-C_1 \tau} \quad \text{and} \quad C_5 = \dfrac{C_2}{C_1}\left(1 - e^{-C_1 \tau}\right),
    $$
    which are both positive constants.
\end{proof}

In the further development of this section, we need the following measure of contraction of the total variation distance of probability measures induced by Markov transitions:   
\begin{definition}[Dobrushin ergodic coefficient]
    For any Markov transition $T$ and a Lyapunov function $W$ with $W\geq 1$, we define the Dobrushin ergodic coeeficient~\citep[Definition 8.2.11]{del2017stochastic} as
    \begin{equation*}
        \beta(T) = \sup_{\vw_\mX, \vw_\mX' \in \R^n} \|T(\vw_\mX, \cdot) - T(\vw_\mX', \cdot)\|_\TV,\
    \end{equation*}
    where $\|\cdot\|_\TV$ is the total variation distance between two probability measures. We also define the
    $W$-Dobrushin ergodic coefficient~\citep[Definition 8.2.19]{del2017stochastic} as
    \begin{equation*}
        \beta_W(T) = \sup_{\vw_\mX, \vw_\mX' \in \R^n} \dfrac{\|T(\vw_\mX, \cdot) - T(\vw_\mX', \cdot)\|_{W}}{1 + W(\vw_\mX) + W(\vw_\mX')},
    \end{equation*}
    where $\|\cdot\|_W$ is the $W$-norm defined as 
    \begin{equation*}
        \|f\|_W = \sup_{\vw_\mX \in \R^n} \dfrac{|f(\vw_\mX)|}{1/2 + W(\vw_\mX)}.
    \end{equation*}
\end{definition}

The $W$-Dobrushin ergodic coefficient $\beta_W(T)$ has the following properties:
\begin{proposition}
    For any Markov transitions $T$, $T_1$, and $T_2$, and any measures $\mu_1$ and $\mu_2$ on $R^n$, the $W$-Dobrushin ergodic coefficient $\beta_W(T)$ satisfies:
    \begin{equation}
        \beta_W(T_1 T_2) \leq \beta_W(T_1) \beta_W(T_2),
    \end{equation}
    and 
    \begin{equation}
        \|\mu_1 T - \mu_2 T\|_V \leq \beta_V(T) \|\mu_1 - \mu_2\|_V.
    \end{equation}
\end{proposition}

The following proposition is a direct consequence of the property of Gaussian noise in~\eqref{eq:1_sde_proj}:
\begin{proposition}
    For any time interval $\tau\geq0$,  the following Dobrushin local contraction condition is satisfied by the time-inhomogeneous Markov transition semigroup $P_s^{s+\tau}$ uniformly in time $s$:
    \begin{equation*}
        \beta(P_s^{s+\tau}; C) := \sup_{\vw_\mX, \vw_\mX' \in C} \|P_s^{s+\tau}(\vw_\mX, \cdot) - P_s^{s+\tau}(\vw_\mX', \cdot)\|_\TV<1,
    \end{equation*}
    for any compact subset $C\subset \R^n$.
\label{lem:local_contraction}
\end{proposition}
\begin{proof}
    For any compact subset $C\subset \R^n$, we choose a sufficiently large $R$ such that we have 
    \begin{equation*}
        \langle \overline{\vw_\mX}, \vb(s, \vw_\mX) \rangle <0, \quad \forall s,\ \|\vw_\mX\|>R
    \end{equation*}
    which is attainable because of~\Eqref{eq:lyapunov_2}.

    Fix $\|\vw(s)\|\leq R$ and notice that the norm of $\vw(s)$ satisfies the following SDE:
    \begin{equation*}
        \rd \|\vw_\mX(s)\| = \left(\langle \overline{\vw_\mX(s)}, \vb(s, \vw_\mX(s)) \rangle  + \dfrac{\eta_L}{2}\tr \left(\dfrac{\mI -  \overline{\vw_\mX(s)}\ \overline{\vw_\mX(s)}^\top}{\|\vw_\mX(s)\|} \right)\right) \rd s + \sqrt{\eta_L} \langle \overline{\vw_\mX(s)}, \rd \mB_\mX(s) \rangle,
    \end{equation*}
    we have 
    \begin{equation*}
        \begin{aligned}
            \E[\|\vw_\mX(s+\tau)\|]=&\|\vw_\mX(s)\| + \int_s^{s+\tau} \E\left[\langle \overline{\vw_\mX(u)}, \vb(u, \vw_\mX(u)) \rangle  + \dfrac{\eta_L}{2}\tr \left(\dfrac{\mI -  \overline{\vw_\mX(u)}\ \overline{\vw_\mX(u)}^\top}{\|\vw_\mX(u)\|} \right)\right] \rd u, \\
            \leq& R + \int_s^{s+\tau} {\bm 1}(\|\vw_\mX(u)\|> R)\left[\langle \overline{\vw_\mX(u)}, \vb(u, \vw_\mX(u)) \rangle  + \dfrac{\eta_L}{2}\dfrac{n - 1}{\|\vw_\mX(u)\|}\right] \rd u, \\
            \leq & R + \dfrac{\tau \eta_L(n-1)}{2R},
        \end{aligned}
    \end{equation*}
    and 
    \begin{equation*}
        \Var[\|\vw_\mX(s+\tau)\|] = \int_s^{s+\tau} \left(\sqrt{\eta_L} \langle \overline{\vw_\mX(s)}, \rd \mB_\mX(s) \rangle\right)^2 = \eta_L \tau.
    \end{equation*}

    Therefore, for $\epsilon>0$, it holds that 
    \begin{equation*}
        \|\vw_\mX(s+\tau)\| \leq R + \dfrac{\tau \eta_L(n-1)}{2R}, 
    \end{equation*} 
    with positive probability for any $s$ and $\vw_\mX(s) \in C$.

    Therefore, by the isotropic property of the Gaussian noise, we have for any Borel set $A\subset \R^n$, there exists $\epsilon_A>0$ such that
    \begin{equation*}
        \sP(\vw_\mX(s+\tau)\in A) \geq \epsilon_A \nu(A),
    \end{equation*}
    where $\nu$ is the Lebesgue measure on $\R^n$, which implied the local Dobrushin contraction condition~\citep[Proposition 8.2.18]{del2017stochastic}.
\end{proof}

With the above definition, we have the following lemma:
\begin{lemma}
    For any time interval $\tau\geq0$, there exists a positive function $V$ such that $\beta_V(P_s^{s+\tau})<1$ holds uniformly for any time $s$, \emph{i.e.}
    \begin{equation}
        \beta_V(P_s^{s+\tau}) \leq e^{-\kappa \tau}, \quad \forall s.
    \label{eq:lyapunov_3}
    \end{equation}
\end{lemma}
\begin{proof}
    As shown in Lemma~\ref{lem:local_contraction} and Corollary~\ref{cor:lyapunov}, for any time interval $\tau\geq0$, the time-inhomogeneous Markov transition semigroup $P_{s}^{s+\tau}$ satisfies both the Dobrushin local contraction condition and the Foster-Lyapunov condition w.r.t. the Lyapunov function $W$~\eqref{eq:lyapunov}, uniformly in time $s$. Therefore, by~\citet[Theorem 8.2.21]{del2017stochastic}, there exists a positive function $V$ such that $\beta_V(T)<1$ uniformly for any $s$. Then we take $$\kappa = -\dfrac{1}{\tau}\log \sup_s \beta_V(P_s^{s+\tau}),$$ 
    \Eqref{eq:lyapunov_3} is thus satisfied.

    Without loss of generality, we assume $N = T/\tau\in\sZ$.
    Notice that for any measures $\mu_1$ and $\mu_2$ on $\R^n$, we have
    \begin{equation*}
        \begin{aligned}
            \|\mu_1 P_0^t - \mu_2 P_0^t\|_\TV \leq& \beta_V(P_0^t) \|\mu_1 - \mu_2\|_\TV,\\
            \leq & \prod_{i=0}^{N-1} \beta_V(P_{i\tau}^{(i+1)\tau}) \|\mu_1 - \mu_2\|_\TV,\\
            \leq & e^{-\kappa t} \|\mu_1 - \mu_2\|_\TV \to 0, \quad \text{as}\ t\to\infty,
        \end{aligned}
    \end{equation*}
    which implies~\Eqref{eq:lyapunov_3}.
\end{proof}

We are now ready to prove Theorem~\ref{thm:phase_I} in the main text:
\begin{proof}[Proof of Theorem~\ref{thm:phase_I}]
    
    By the definition of the $V$-Dobrushin ergodic coefficient, for any $f$ satisfying $|f(\vw_\mX)|\leq 1/2 + V(\vw_\mX)$, we have 
    \begin{equation*}
        |P_0^{t}f (\vw_\mX) - P_0^t f (\vw_\mX') |\leq \beta_V(P_0^{t}) \left(1 + V(\vw_\mX) + V(\vw_\mX')\right) \leq e^{-\kappa t} \left(1 + V(\vw_\mX) + V(\vw_\mX')\right),
    \end{equation*}
    for any $t\geq 0$ and $\vw_\mX, \vw_\mX'\in\R^n$.

    Then, by the definition of the total variation distance, we have
    \begin{equation*}
        \begin{aligned}
            \| P_0^t(\vw_\mX, \cdot) - P_0^t f (\vw_\mX')\|_\TV =& \sup_{|f(\vw_\mX)|\leq 1/2} \int_{\R^n} \left(P_0^tf(\vw_\mX) - P_0^t f (\vw_\mX')\right) \mu(\rd \vw_\mX)\\
            \leq & \sup_{|f(\vw_\mX)|\leq 1/2 + V(\vw_\mX)} \int_{\R^n} |P_0^tf(\vw_\mX) - P_0^t f (\vw_\mX')| \mu(\rd \vw_\mX)\\
            \leq & e^{-\kappa t} \left(1 + V(\vw_\mX) +V(\vw_\mX') \right),
        \end{aligned}
    \end{equation*}
    and the statement follows by taking $\vw_\mX = \gP_\mX \vw_0$ and $\vw_\mX' = \gP_\mX \vw_0'$.
\end{proof}

\section{MISSING PROOFS FOR PHASE II}
\label{app:phase_II}

In this section, we provide the missing proofs for the results in Phase II in Section~\ref{sec:phase_II} of the main text. For convenience, we will assume the time $t$ is reset to $0$ after Phase I and thus the initial condition $\vw(0)$ of the SDE~\eqref{eq:2_sde} is already near the minima manifold $\gM$, as shown in Theorem~\ref{thm:phase_I}

Before we dive into the discussion of the effective dynamics of $\vw(t)$ along the minima manifold $\gM$, we first introduce the following lemma that characterizes the space around $\vw_\gM\in\gM$:
\begin{proposition}
    For any $\vw_\gM\in\gM$, the normal space of the manifold $\gM$ around $\vw_\gM$ is given by 
    \begin{equation}
        \gN(\vw_\gM; \gM) = \vw_\gM + \mA_\gM \mX ,
        \label{eq:normal_space}
    \end{equation}
    and the tangent space of the manifold $\gM$ around $\vw_\parallel$ is given by
    \begin{equation}
        \gT(\vw_\gM; \gM) = \vw_\gM + \mA_\gM^{-1} \mX^\perp ,
        \label{eq:tangent_space}
    \end{equation}
    where $\mA_\gM = \mA(\vw_\gM)$. 
    \label{prop:space}
\end{proposition}
\begin{proof}
    For simplicity, we will also adopt the notation $\valpha_\gM = \valpha(\vw_\gM)$.

    Consider the following expansion of the gradient $\nabla_\vw \hat{\gL}(\vw; \vw^*, \gD)$ around $\vw_\gM$ up to the first order:
    \begin{equation}
        \begin{aligned}
            \nabla_\vw \hat{\gL}(\vw; \vw^*, \gD) &=  \dfrac{1}{n} \left(\mA_\gM + O(\|\vw - \vw_\gM\|)\right) \mX\mX^\top (\valpha - \valpha^*) \\
            & = \dfrac{1}{n} \left(\mA_\gM + O(\|\vw - \vw_\gM\|)\right) \mX\mX^\top (\valpha - \valpha_\gM)\\
            & = \dfrac{1}{n} \mA_\gM \mX\mX^\top (\valpha - \valpha_\gM)+ o(\|\vw - \vw_\gM\|)
        \end{aligned}
        \label{eq:L_expansion}
    \end{equation}
    where the second to last equality is due to $\valpha_\gM - \valpha^*\in \mX^\perp$ and therefore $\mX^\top (\valpha_\infty - \valpha^*)=\vzero$, and the last equality is because $$\|\valpha - \valpha_\infty\|\asymp \|\nabla_\vw\valpha(\vw_\gM)^\top(\vw-\vw_\gM)\|\asymp \|\vw-\vw_\gM\|,$$
    where $\asymp$ denotes the equivalence up to a constant factor.

    Let $\valpha - \valpha_\gM = n\mX (\mX^\top \mX)^{-1}\veps$, where $\veps\in\R^n$, we have 
    $$
        \nabla_\vw \hat{\gL}(\vw; \vw^*, \gD)  = \mA_\gM \mX \veps + o(\|\vw - \vw_\gM\|) =  \mA_\gM \mX \veps + o(\|\veps\|),
    $$
    and \Eqref{eq:normal_space} follows by taking $\|\veps\| \to 0$. Then, it is straightforward to see \Eqref{eq:tangent_space} holds by noticing $\mA_\gM$ is invertible.
\end{proof}

\subsection{Proof of Lemma~\ref{lem:effective_dynamics}}

For readers' convenience, we restate the form of the SDE that we are considering in Phase II:
\begin{equation*}
    \rd \vw(t) = - \nabla_\vw \hat{\gL}(\vw(t); \vw^*, \gD) + \sqrt{\eta_L} \msigma(\vw(t)) \rd \mB(t).
\end{equation*}

\begin{proof}[Proof of~\Eqref{eq:2_sde_proj}]

  The dynamics in the tangent space $\gT(\vw_\gM; \gM)$ are given by
    \begin{equation*}
        \begin{aligned}
            \rd \Delta \vw_\parallel(t) &= - \gP_{\gT(\vw_\gM; \gM)} \nabla_\vw \hat{\gL}(\vw(t); \vw^*, \gD) \rd t + \gP_{\gT(\vw_\gM; \gM)} \sqrt{\eta_L} \msigma(\vw(t)) \rd \mB(t),\\
            &= - \gP_{\mA_\gM^{-1}\mX^\perp} \nabla_\vw \hat{\gL}(\vw(t); \vw^*, \gD) \rd t + \gP_{\mA_\gM^{-1}\mX} \sqrt{\eta_L} \msigma(\vw(t)) \rd \mB(t),\\
            &= - \gP_{\mA_\gM^{-1}\mX^\perp}  \nabla_\vw \hat{\gL}(\vw(t); \vw^*, \gD) \rd t + \sqrt{\eta_L} \msigma_\parallel(\vw(t)) \rd \mB(t),
        \end{aligned}
    \end{equation*}
    and the case for the dynamics in the normal space $\gN(\vw_\gM; \gM)$ is similar.
\end{proof}

In order to analyze the empirical loss landscape $\hat{\gL}(\vw; \vw^*, \gD)$ locally around $\vw_\gM$, we denote $\Delta \vw = \Delta \vw_\parallel + \Delta \vw_\perp = \vw - \vw_\gM$ and perform the following calculation on the expansion of $\mA(\vw)$ around $\vw_\gM$ up to the first order:
$$
\begin{aligned}
    \nabla \mA[\Delta \vw](\vw_\gM)=\gamma\|\vw_\gM\|^{\gamma-1} \left(\overline{\vw_\gM}^\top \Delta \vw\right) \mI &+ \gamma(\gamma-2)\|\vw_\gM\|^{\gamma-1}\left(\overline{\vw_\gM}^\top \Delta \vw\right) \overline{\vw_\gM}\ \overline{\vw_\gM}^\top \\&+\gamma\|\vw_\gM\|^{\gamma-1}\left( \Delta \vw \overline{\vw_\gM}^\top +\overline{\vw_\gM} \left(\Delta \vw\right)^\top\right) + o(\|\Delta \vw\|),
\end{aligned}
$$
where we adopt the notation of the directional derivative:
\begin{equation}
    \nabla\vf[\vv](\vx) = \lim_{\epsilon\to 0}\dfrac{\vf(\vx + \epsilon \vv) - f(\vx)}{\epsilon},
    \label{eq:directional_derivative}
\end{equation}
for any $\vv$ of the same dimension as $\vx$.

Also, by the definition of the matrix $\mA$ in~\Eqref{eq:gradient}, we have the expansion of $\alpha(\vw)$ around $\vw_\gM$ up to the second order:
\begin{equation*}
    \valpha - \valpha_\gM = \mA_\gM \Delta \vw + \dfrac{1}{2}\nabla_{\vw}^2 \valpha[\Delta \vw,\Delta \vw] (\vw_\gM) + o(\|\Delta \vw\|^2),
\end{equation*}
where the notation $\nabla_{\vw}^2 \valpha[\Delta \vw,\Delta \vw] (\vw_\gM)$ is defined analogously to the directional derivative in~\Eqref{eq:directional_derivative}.

Then we have the following expansion of the gradient $\nabla_\vw \hat{\gL}(\vw; \vw^*, \gD)$ around $\vw_\gM$ up to the second order:
\begin{equation}
    \begin{aligned}
        &\nabla_\vw \hat{\gL}(\vw; \vw^*, \gD) = \nabla_\vw \hat{\gL}(\vw_\gM + \Delta \vw; \vw^*, \gD) \\
        =& \dfrac{1}{n} \big(\mA(\vw_\gM) + \nabla \mA[\Delta \vw](\vw_\gM) + o(\|\Delta \vw\|) \big)\mX\mX^\top (\valpha - \valpha_\gM)\\
        =&  \dfrac{1}{n} \bigg(\mA_\gM + \gamma\|\vw_\gM\|^{\gamma-1} \left(\overline{\vw_\gM}^\top \Delta \vw\right) \mI + \gamma(\gamma-2)\|\vw_\gM\|^{\gamma-1}\left(\overline{\vw_\gM}^\top \Delta \vw\right) \overline{\vw_\gM}\ \overline{\vw_\gM}^\top\\
        &+ \gamma\|\vw_\gM\|^{\gamma-1}\left( \Delta \vw \overline{\vw_\gM}^\top +  \overline{\vw_\gM} \left(\Delta \vw\right)^\top\right) +o(\|\Delta \vw\|)\bigg) \mX\mX^\top \bigg(\mA_\gM \Delta \vw  + \dfrac{1}{2}\nabla_{\vw}^2 \valpha[\Delta \vw,\Delta \vw] (\vw_\gM) + o(\|\Delta \vw\|^2)\bigg)\\
        =&\dfrac{1}{n}\mA_\gM\mX\mX^\top \mA_\gM\Delta \vw + \dfrac{1}{n}\|\vw_\gM\|^{\gamma-1}\bigg( \gamma \mX\mX^\top \mA_\gM \Delta \vw (\Delta \vw)^\top \overline{\vw_\gM} + \gamma(\gamma-2)\overline{\vw_\gM}\ \overline{\vw_\gM}^\top \mX\mX^\top \mA_\gM \Delta \vw (\Delta \vw)^\top \overline{\vw_\gM}\\
        &+ \gamma\left(\Delta \vw (\Delta \vw)^\top \mA_\gM \mX \mX^\top\overline{\vw_\gM} +  \overline{\vw_\gM} (\Delta \vw)^\top \mX\mX^\top \mA_\gM \Delta \vw \right)\bigg) + \dfrac{1}{2n} \mA_\gM \mX\mX^\top \nabla_{\vw}^2 \valpha[\Delta \vw,\Delta \vw] (\vw_\gM)  + o(\|\Delta \vw\|^2).
    \end{aligned}
\label{eq:2_gradient_expansion}
\end{equation}

As a direct corollary of the above expansion, we obtain the Hessian matrix of the empirical loss $\hat{\gL}(\vw; \vw^*, \gD)$ around $\vw_\gM$:
\begin{equation*}
    \hat{\gL}(\vw_\gM; \vw^*, \gD) = \dfrac{1}{n}\mA_\gM\mX\mX^\top \mA_\gM.
\end{equation*}

Following the reasoning in the main text, we adopt the label noise model~\citep{blanc2020implicit,damian2021label} and the following assumption on the diffusion tensor as in~\eqref{eq:2_assumption}:
\begin{equation}
    \msigma(\vw(t))\msigma(\vw(t))^\top := \sigma^2 \nabla^2 \hat{\gL}(\vw(t); \vw^*, \gD)\approx \sigma^2 \nabla^2 \hat{\gL}(\vw_\gM; \vw^*, \gD) = \dfrac{\sigma^2}{n} \mA_\gM \mX \mX^\top \mA_\gM,
\label{eq:2_sigma}
\end{equation}
where the approximation is due to the fact that $\vw(t)$ is close to $\vw_\gM$. A natural and simple solution to~\Eqref{eq:2_sigma} is to take
\begin{equation*}
    \msigma(\vw(t)) := \dfrac{\sigma}{\sqrt{n}}\left[\begin{matrix}
        \vzero & \mA_\gM \mX
    \end{matrix}\right],
\end{equation*}
where $\vzero$ is a zero matrix of size $d\times (d-n)$, by which the diffusion term of the SDE~\eqref{eq:2_sde} is simplified as
\begin{equation}
    \msigma(\vw(t)) \rd \mB(t) = \dfrac{\sigma}{\sqrt{n}}\mA_\gM \mX \rd \mB_\perp(t),
\label{eq:2_sde_diffusion}
\end{equation}
where $\rd \mB_\perp(t)$ is a $n$-dimensional Brownian motion. One can verify that with this diffusion tensor, we have
\begin{equation}
    \msigma_{\parallel}(t) = \vzero, \quad \msigma_{\perp}(t) = \dfrac{\sigma}{\sqrt{n}}\mA_\gM \mX,
\label{eq:2_sde_diffusion_2}
\end{equation}
which intuitively means that the label noise only affects the normal space of the manifold $\gM$. This is in accordance with the discussions in~\citet{li2021happens,ma2022beyond}.


\begin{proof}[Proof of~\Eqref{eq:2_ou} and~\Eqref{eq:2_stationary}]
    Plugging the diffusion tensor~\eqref{eq:2_sde_diffusion_2} and the first order expansion of the empirical loss at the small scale $\hat{\gL}(\vw(t+\tau); \vw^*, \gD)$ into the dynamics of $\Delta\vw_\perp(t)$ in~\eqref{eq:2_sde_proj}, we have
    \begin{equation*}
        \begin{aligned}
            \rd \Delta \vw_\perp(t+\tau) &= - \dfrac{1}{n}\mA_\gM\mX\mX^\top \mA_\gM\Delta \vw(t+\tau) \rd \tau + \sqrt{\eta_L} \msigma_{\perp}(\vw(t+\tau)) \rd \mB(t), \\
            &= - \dfrac{1}{n}\mA_\gM\mX\mX^\top \mA_\gM\Delta \vw_\perp(t+\tau) \rd \tau + \sqrt{\dfrac{\sigma^2\eta_L}{n}}\mA_\gM \mX \rd \mB_\perp(\tau),
        \end{aligned} 
    \end{equation*}
    where the second equality is due to $\mX^\top \mA_\gM \Delta \vw = \mX^\top \mA_\gM \gP_{\mA_\gM \mX} \Delta \vw = \mX^\top \mA_\gM \Delta \vw_\perp$. 

    Due to the local representation of the normal space in Proposition~\ref{prop:space}, we simplify the above dynamics by taking the parametrization $\Delta \vw_\perp = \mA_\gM \mX \veps$, where $\veps\in\R^n$, and obtain 
    \begin{equation}
        \mA_\gM \mX \rd \veps(\tau) = - \dfrac{1}{n}\mA_\gM\mX\mX^\top \mA_\gM^2 \mX \veps(\tau) \rd \tau + \sqrt{\dfrac{\sigma^2\eta_L}{n}}\mA_\gM \mX \rd \mB_\perp(\tau),
        \label{eq:2_ou_1}
    \end{equation}
    which is an Ornstein-Uhlenbeck process. This parametrization is invertible since $\mA_\gM$ is invertible and $\mX$ is full rank.

    Rewrite~\Eqref{eq:2_ou_1} as 
    \begin{equation*}
        \rd \veps(\tau) = - \dfrac{1}{n}\mX^\top \mA_\gM^2 \mX \veps(\tau) \rd \tau + \sqrt{\dfrac{\sigma^2\eta_L}{n}}\rd \mB_\perp(\tau),
    \end{equation*}
    we obtain the OU process as in~\Eqref{eq:2_ou}.

    As a classical result, the stationary distribution of the above OU process is given by $\veps(\tau) \sim \gN\left(\vzero, \mSigma \right)$, where $\mSigma$ satisfies
    \begin{equation*}
        \dfrac{1}{n} \mX^\top \mA_\gM^2 \mX \mSigma + \mSigma \dfrac{1}{n} (\mX^\top \mA_\gM^2 \mX)^\top = \dfrac{\sigma^2\eta_L}{n} \mI, \quad\text{\emph{i.e.}}\quad\mSigma = \dfrac{\sigma^2\eta_L}{2} \left(\mX^\top \mA_\gM^2 \mX\right)^{-1},
    \end{equation*}
    and the claim in~\Eqref{eq:2_stationary} follows.
\end{proof}

Now we are ready to present the proof of Lemma~\ref{lem:effective_dynamics} in the main text:
\begin{proof}[Proof of Lemma~\ref{lem:effective_dynamics}]

    From the stationary distribution~\eqref{eq:2_stationary}, we have 
    \begin{equation}
        \E_{\Delta\vw_\perp(t)}[\Delta \vw_\perp] = \vzero, \quad \text{and} \quad \E_{\Delta\vw_\perp(t)}[\Delta \vw_\perp \Delta \vw_\perp^\top] = \dfrac{\sigma^2\eta_L}{2} \mA_\gM \mX \left(\mX^\top \mA_\gM^2 \mX\right)^{-1}\mX^\top\mA_\gM,
        \label{eq:exp_delta_vw}
    \end{equation}
    where $\E_{\Delta\vw_\perp(t)}$ denotes taking expectation w.r.t. the stationary distribution of $\Delta \vw_\perp(t)$~\eqref{eq:2_stationary}.

    Plugging~\Eqref{eq:exp_delta_vw} into~\Eqref{eq:2_gradient_expansion}, we have 
    \begin{equation}
        \begin{aligned}
            &\E_{\Delta\vw_\perp(t)}[\nabla_\vw \hat{\gL}(\vw; \vw^*, \gD)]\\
            =& \dfrac{\sigma^2 \eta_L}{2n}\|\vw_\gM\|^{\gamma-1}\bigg( \gamma \mX \mX^\top \mA_\gM  \overline{\vw_\gM} + \gamma(\gamma-2)\overline{\vw_\gM}\ \overline{\vw_\gM}^\top \mX \mX^\top \mA_\gM  \overline{\vw_\gM} + \gamma\left(\mA_\gM \mX \mX^\top\overline{\vw_\gM} +   \tr(\mX^\top \mA_\gM \mX) \overline{\vw_\gM}\right)\bigg)\\
            &+ \dfrac{\sigma^2 \eta_L}{4n} \mA_\gM \mX\mX^\top \E_{\Delta\vw_\perp(t)}\left[\nabla_{\vw}^2 \valpha[\Delta \vw_\perp,\Delta \vw_\perp] (\vw_\gM)\right] + o(\|\Delta\vw_\perp(t)\|^2).
        \end{aligned}
        \label{eq:expgrad}
    \end{equation}

    The effective force along the tangent space up to the second order is thus obtained by projecting \Eqref{eq:expgrad} onto the tangent space, causing all terms in the column space of $\mA_\gM \mX$ vanishing:
    \begin{equation}
        \begin{aligned}
            &-\gP_{ \mA_\gM^{-1}\mX^\perp}\E_{\Delta\vw_\perp(t)}[\nabla_\vw \hat{\gL}(\vw; \vw^*, \gD)] \\
            =& -\gP_{ \mA_\gM^{-1}\mX^\perp}\dfrac{\sigma^2 \eta_L\gamma}{2n}\|\vw_\gM\|^{\gamma-1} \bigg(\mX \mX^\top \mA_\gM  \overline{\vw_\gM} + (\gamma-2)\overline{\vw_\gM}\ \overline{\vw_\gM}^\top \mX \mX^\top \mA_\gM  \overline{\vw_\gM} +   \tr(\mX^\top \mA_\gM \mX) \overline{\vw_\gM}\bigg)\\
            =& -\gP_{ \mA_\gM^{-1}\mX^\perp}\dfrac{\sigma^2 \eta_L\gamma}{2n}\|\vw_\gM\|^{\gamma-1} \bigg( (1+\gamma)\|\vw_\gM\|^\gamma \mX \mX^\top \overline{\vw_\gM}+ (\gamma-2)(1+\gamma)\|\vw_\gM\|^\gamma \left(\overline{\vw_\gM}^\top \mX \mX^\top   \overline{\vw_\gM}\right) \overline{\vw_\gM} \\
            & +  \|\vw_\gM\|^\gamma \left( \tr(\mX^\top \mX) + \gamma \overline{\vw_\gM}^\top \mX \mX^\top   \overline{\vw_\gM} \right) \overline{\vw_\gM}\bigg)\\
            =&-\gP_{ \mA_\gM^{-1}\mX^\perp}\dfrac{\sigma^2 \eta_L\gamma}{2n}\|\vw_\gM\|^{2\gamma-1} \bigg( (1+\gamma) \mX \mX^\top \overline{\vw_\gM} + \tr(\mX^\top \mX)\overline{\vw_\gM} + (\gamma^2-2)\left(\overline{\vw_\gM}^\top \mX \mX^\top   \overline{\vw_\gM}\right) \overline{\vw_\gM} \bigg),
        \end{aligned}
        \label{eq:expgrad2}
    \end{equation}
    where the second to last equality is due to 
    $$
    \mA_\gM \overline{\vw_\gM} = \|\vw_\gM\|^\gamma \left(\mI + \gamma \overline{\vw_\gM}\ \overline{\vw_\gM} ^\top\right) \overline{\vw_\gM} = (1+\gamma)\|\vw_\gM\|^\gamma\overline{\vw_\gM}.
    $$

    By the Sherman-Morrison formula, we have  
    \begin{equation}
        \mA_\gM^{-1} = \|\vw_\gM\|^{-\gamma} \left(\mI - \frac{\gamma}{1+\gamma} \overline{\vw_\gM}\ \overline{\vw_\gM} ^\top\right),
    \label{eq:sherman_morrison}
    \end{equation}
    and thus 
    $$
    \begin{aligned}
        \mA_\gM^{-1} \overline{\vw_\gM} &= \|\vw_\gM\|^{-\gamma} \left(\mI - \dfrac{\gamma}{1+\gamma} \overline{\vw_\gM}\ \overline{\vw_\gM} ^\top\right) \overline{\vw_\gM} = \dfrac{1}{1+\gamma}\|\vw_\gM\|^{-\gamma}\overline{\vw_\gM}\\
        \mA_\gM^{-1} \mX \mX^\top\overline{\vw_\gM} &= \|\vw_\gM\|^{-\gamma} \left(\mX \mX^\top\overline{\vw_\gM} - \dfrac{\gamma}{1+\gamma} (\overline{\vw_\gM} ^\top \mX \mX^\top \overline{\vw_\gM})\overline{\vw_\gM}\right),
    \end{aligned}
    $$
    and recall 
    $$
    \gP_{ \mA_\gM^{-1}\mX^\perp} = \mA_\gM^{-1}\mX^\perp ((\mA_\gM^{-1}\mX^\perp)^\top \mA_\gM^{-1}\mX^\perp)^{-1}(\mA_\gM^{-1}\mX^\perp)^\top = \mA_\gM^{-1}\mX^\perp ((\mX^\perp)^\top\mA_\gM^{-2}\mX^\perp)^{-1} (\mX^\perp)^\top \mA_\gM^{-1},
    $$
    by which \Eqref{eq:expgrad2} can be further simplified as
    $$
    \begin{aligned}
        &-\gP_{ \mA_\gM^{-1}\mX^\perp}\dfrac{\sigma^2 \eta_L\gamma}{2n}\|\vw_\gM\|^{2\gamma-1} \bigg( (1+\gamma) \mX \mX^\top \overline{\vw_\gM} + \tr(\mX^\top \mX)\overline{\vw_\gM} + (\gamma^2-2)\left(\overline{\vw_\gM}^\top \mX \mX^\top   \overline{\vw_\gM}\right) \overline{\vw_\gM} \bigg)\\
        =&-\dfrac{\sigma^2 \eta_L\gamma}{2n}\|\vw_\gM\|^{\gamma-1} \mA_\gM^{-1}\mX^\perp\left(\left(\mX^\perp\right)^\top\mA_\gM^{-2}\mX^\perp\right)\left(\mX^\perp\right)^\top
        \bigg( (1+\gamma)\mX \mX^\top\overline{\vw_\gM} - \gamma \left(\overline{\vw_\gM} ^\top \mX \mX^\top \overline{\vw_\gM}\right)\overline{\vw_\gM}\\
        &+ \dfrac{1}{1+\gamma}\tr(\mX^\top \mX)\overline{\vw_\gM} + \dfrac{\gamma^2-2}{1+\gamma}\left(\overline{\vw_\gM}^\top \mX \mX^\top   \overline{\vw_\gM}\right) \overline{\vw_\gM} \bigg)\\
        =&-\dfrac{\sigma^2 \eta_L\gamma}{2n}\|\vw_\gM\|^{2\gamma-1} \mA_\gM^{-1}\mX^\perp\left(\left(\mX^\perp\right)^\top\mA_\gM^{-2}\mX^\perp\right)\left(\mX^\perp\right)^\top
        \bigg(  - \gamma(1+\gamma) \left(\overline{\vw_\gM} ^\top \mX \mX^\top \overline{\vw_\gM}\right) \mA_\gM^{-1}\overline{\vw_\gM}\\
        &+ \tr(\mX^\top \mX) \mA_\gM^{-1} \overline{\vw_\gM} + (\gamma^2-2)\left(\overline{\vw_\gM}^\top \mX \mX^\top   \overline{\vw_\gM}\right)  \mA_\gM^{-1}\overline{\vw_\gM} \bigg)\\
        =&-\dfrac{\sigma^2 \eta_L\gamma}{2n}\|\vw_\gM\|^{2\gamma-1} 
        \bigg( \tr(\mX^\top \mX)   - (\gamma+2)\left(\overline{\vw_\gM}^\top \mX \mX^\top   \overline{\vw_\gM}\right)  \bigg) \gP_{ \mA_\gM^{-1}\mX^\perp}\overline{\vw_\gM},
    \end{aligned}
    $$
    where the second to last equality is due to $(\mX^\perp)^\top\mX = \vzero$.

    The dynamics of $\Delta \vw_\parallel(t)$ are thus given by
    \begin{equation}
        \begin{aligned}
            \rd \Delta \vw_\parallel(t) =& - \E_{\Delta \vw_\perp(t)}  \left[\gP_{\mA_\gM^{-1}\mX^\perp} \nabla_\vw \hat{\gL}(\vw(t); \vw^*, \gD)\right] \rd t\\
            =& -\dfrac{\sigma^2 \eta_L\gamma}{2n}\|\vw_\gM(t)\|^{2\gamma-1} 
            \bigg( \tr(\mX^\top \mX)   - (\gamma+2)\left(\overline{\vw_\gM}^\top \mX \mX^\top   \overline{\vw_\gM}\right)  \bigg) \gP_{ \mA_\gM^{-1}\mX^\perp}\overline{\vw_\gM}\rd t.
        \end{aligned}
    \label{eq:2_effective_dynamics_1}
    \end{equation}

    Since the above analysis is based on the assumption that $\vw(t)$ is close to $\vw_\gM$, we conclude that the effective dynamics of $\vw_\gM(t)$ up to the second order are given by
    \begin{equation*}
        \rd \vw_\gM(t) \approx \rd \Delta \vw_\parallel(t) = -\dfrac{\sigma^2 \eta_L\gamma}{2n}\|\vw_\gM(t)\|^{2\gamma-1} 
        \bigg( \tr(\mX^\top \mX)   - (\gamma+2)\left(\overline{\vw_\gM(t)}^\top \mX \mX^\top   \overline{\vw_\gM(t)}\right)  \bigg) \gP_{ \mA_\gM^{-1}\mX^\perp}\overline{\vw_\gM(t)}\rd t,
    \end{equation*}
    and the statement in Lemma~\ref{lem:effective_dynamics} follows.
\end{proof}

\begin{remark}
    As an alternative approach, the effective dynamics are in accordance with~\citet[Corollary 5.2]{li2021happens} by taking $c = \sigma^2$, where the authors applied the results in~\citet{katzenberger1990solutions} to analyze the dynamics of the gradient flow on the manifold $\gM$.
\end{remark}

\subsection{Proof of Theorem~\ref{thm:phase_II}}

We restate the definition of the minima manifold $\gM$~\eqref{eq:gM} as follows:
\begin{equation*}
    \gM = \left\{ \vw \big| \valpha - \valpha^* = \|\vw\|^\gamma \vw - \|\vw^*\|^\gamma \vw^* \in \mX^\perp \right\}.
\end{equation*}
In the following, we will also use $\valpha_\mX^*$ to denote the projection of $\valpha^*$ onto the column space of $\mX$, \emph{i.e.} $\valpha_\mX^* = \gP_\mX \valpha^*$, and use $\valpha_{\mX^\perp}^*$ to denote the projection of $\vw^*$ and $\valpha^*$ onto the null space of $\mX$, \emph{i.e.} $\valpha_{\mX^\perp}^* = \gP_{\mX^\perp} \valpha^*$.

We first closely examine the geometry of the minima manifold $\gM$ and summarize in the following proposition:
\begin{proposition}
    For every $\vw_\gM\in\gM$, it has the following representation:
    $$
        \vw_\gM = \lambda \valpha_\mX^* + \sqrt{\lambda^{-2/\gamma} - \lambda^2 \|\valpha_\mX^*\|^2}\overline{\vr_{\mX^\perp}},
    $$
    where $0<\lambda = \|\vw_\gM\|^{-\gamma} \leq \|\valpha_\mX^*\|^{-\frac{\gamma}{1+\gamma}}$ and $\overline{\vr_{\mX^\perp}}$ is an arbitrary unit vector in $\mX^\perp$. Then we have 
    \begin{equation*}
        \gP_\mX \vw_\gM = \lambda \valpha_\mX^*, \quad \text{and} \quad \gP_{\mX^\perp} \vw_\gM = \sqrt{\lambda^{-2/\gamma} - \lambda^2 \|\valpha_\mX^*\|^2}\overline{\vr_{\mX^\perp}}.
    \end{equation*}
    \label{prop:representation}
\end{proposition}
\begin{proof}
    Following the definition of the minima manifold $\gM$, $\vw_\gM\in\gM$ is equivalent to the following equation:
    \begin{equation*}
       \vzero = \mX^\top\left(\|\vw_\gM\|^\gamma \vw_\gM - \valpha^* \right) = \mX^\top\left(\|\vw_\gM\|^\gamma \gP_\mX \vw_\gM - \valpha_\mX^* \right),
    \end{equation*}
    where the second equality is due to $\mX^\top \gP_{\mX^\perp} = \vzero$. Therefore, notice that $\|\vw_\gM\|^\gamma \gP_\mX \vw_\gM - \valpha_\mX^* \in \mX$, we have
    \begin{equation*}
        \gP_\mX \vw_\gM = \|\vw_\gM\|^{-\gamma}\valpha_\mX^*:=\lambda \alpha_\mX^*,
    \end{equation*}
    where $\lambda = \|\vw_\gM\|^{-\gamma}>0$. 

    Moreover, since 
    \begin{equation*}
        \|\vw_\gM\|^2 = \|\gP_\mX \vw_\gM\|^2 + \|\gP_{\mX^\perp} \vw_\gM\|^2 = \lambda^2 \|\valpha_\mX^*\|^2 + \|\gP_{\mX^\perp} \vw_\gM\|^2,
    \end{equation*}
    we have 
    \begin{equation*}
        \|\gP_{\mX^\perp} \vw_\gM\|^2 = \|\vw_\gM\|^2 - \lambda^2 \|\valpha_\mX^*\|^2 = \lambda^{-2/\gamma} - \lambda^2 \|\valpha_\mX^*\|^2,
    \end{equation*}
    and thus $\gP_{\mX^\perp} \vw_\gM$ can be represented as $$\gP_{\mX^\perp} \vw_\gM = \sqrt{\lambda^{-2/\gamma} - \lambda^2 \|\valpha_\mX^*\|^2}\overline{\vr_{\mX^\perp}},$$ where $\overline{\vr_{\mX^\perp}}$ can take any unit vector in $\mX^\perp$. The statement in the lemma follows.
\end{proof}

Then we can prove the following lemma, which is the key to the proof of Theorem~\ref{thm:phase_II}:
\begin{lemma}
    The optimization problem that finds the minimum $L^2$-norm solution of the empirical loss $\hat{\gL}(\vw; \vw^*, \gD)$, \emph{i.e.}
    \begin{equation*}
        \min_{\vw} \|\vw_\gM\|, \quad \text{s.t.} \quad \hat \gL(\vw; \vw^*, \gD) = 0,
    \end{equation*}
    or equivalently 
    \begin{equation}
        \min_{\vw_\gM} \dfrac{1}{2}\|\vw_\gM\|^2, \quad \text{s.t.} \quad \mX^\top\left(\|\vw_\gM\|^\gamma \vw_\gM - \valpha^*\right) = \vzero,
    \label{eq:2_min}
    \end{equation}
    has a unique local minimum, and thus the global minimum, which is given by
    \begin{equation*}
        \vw^\dagger = \|\valpha_\mX^*\|^{-\frac{\gamma}{1+\gamma}} \valpha_\mX^*,\quad{with optimal value}\quad \|\vw^\dagger\| = \|\valpha_\mX^*\|^{\frac{1}{1+\gamma}} 
    \end{equation*}
    One should notice that $\vw^\dagger$ is in the column space of $\mX$, \emph{i.e.} $\vw^\dagger = \gP_\mX \vw^\dagger$.
\end{lemma}
\begin{proof}
    Apply the KKT conditions to the above constrainted optimization problem~\eqref{eq:2_min}, we have for $\vmu\in \R^n$ such that 
    \begin{equation*}
        \begin{aligned}
            \vzero &= \nabla_{\vw} \left[\dfrac{1}{2}\|\vw_\gM\|^2 + \vmu^\top \mX^\top \left(\|\vw_\gM\|^\gamma \vw_\gM - \valpha^*\right) \right]\\
            &= \vw_\gM + \nabla_\vw\valpha_\gM \mX\vmu,\\
            &= \vw_\gM + \left(\|\vw_\gM\|^\gamma \mI + \gamma \|\vw_\gM\|^{\gamma - 2 } \vw_\gM \vw_\gM^\top\right)\mX\vmu.
        \end{aligned}
    \end{equation*}

    Then consider the projection of the above equation onto the column space of $\mX$ and its null space, and plug in Proposition~\ref{prop:representation}, we have
    \begin{equation}
        \begin{cases}
            \lambda \valpha_\mX^* + \lambda^{-1} \mX\vmu + \gamma \lambda^{-1 + 2/\gamma} \lambda \valpha_\mX^* \left(\lambda \valpha_\mX^*\right)^\top \mX\vmu &= \vzero,\\
            \gP_{\mX^\perp} \vw_\gM + \gamma \lambda^{-1 + 2/\gamma} \gP_{\mX^\perp} \vw_\gM \left(\lambda \valpha_\mX^*\right)^\top \mX\vmu&= \vzero.
        \end{cases}
        \label{eq:2_kkt}
    \end{equation}

    From the first equation of~\Eqref{eq:2_kkt}, we see $\mX\vmu \propto \valpha_\mX^*$, whereby we set $\mX\vmu = \kappa \valpha_\mX^*$, and obtain
    \begin{equation*}
        \lambda \valpha_\mX^* + \lambda^{-1} \kappa \valpha_\mX^* + \gamma \lambda^{1 + 2/\gamma} \kappa \left(\valpha_\mX^*\right)^\top \valpha_\mX^* \valpha_\mX^* = \vzero,
    \end{equation*}
    which solves $\kappa$ as 
    \begin{equation*}
        \kappa = -\dfrac{\lambda^{2}}{1 + \lambda^{2+2/\gamma} \gamma \|\valpha_\mX^*\|^2}.
    \end{equation*}

    Then the second equation of~\Eqref{eq:2_kkt} becomes
    \begin{equation*}
        \left(1 + \gamma \lambda^{2/\gamma} \kappa \|\valpha_\mX^*\|^2 \right)\gP_{\mX^\perp} \vw_\gM =\vzero.
    \end{equation*}
    However, since 
    \begin{equation*}
        1 + \gamma \lambda^{2/\gamma} \kappa \|\valpha_\mX^*\|^2 = 1 -  \dfrac{\lambda^{2+2/\gamma}\gamma \|\valpha_\mX^*\|^2}{1 + \lambda^{2+2/\gamma} \gamma \|\valpha_\mX^*\|^2} = \dfrac{1}{1 + \lambda^{2+2/\gamma} \gamma \|\valpha_\mX^*\|^2} > 0,
    \end{equation*}
    we must have $\gP_{\mX^\perp} \vw_\gM = \vzero$, which means $\vw_\gM = \lambda \valpha_\mX^*$. By Proposition~\ref{prop:representation}, we have
    \begin{equation*}
        \lambda^{-1/\gamma} = \|\vw_\gM\| =\lambda \|\valpha_\mX^*\|,
    \end{equation*}
    and thus 
    \begin{equation*}
        \lambda = \|\valpha_\mX^*\|^{-\frac{\gamma}{1+\gamma}}.
    \end{equation*}
    
    It is straightforward to verify that the linear independence constraint qualification (LICQ) is always satisfied, and thus $\vw^\dagger = \|\valpha_\mX^*\|^{-\frac{\gamma}{1+\gamma}} \valpha_\mX^*$ is the unique local minimum of the optimization problem~\eqref{eq:2_min}, and therefore the global minimum.
\end{proof}

The following lemma is drawn from the random matrix theory:
\begin{lemma}[{\citet[Theorem 4.4.5]{vershynin2018high}}]
    Let $\mA$ by an $m\times n$ random matrix with i.i.d. entries drawn from $\gN(0,1)$. Then for any $t>0$, we have 
    \begin{equation*}
        \|\mA\| \leq \gO(\sqrt{m} + \sqrt{n} + t),
    \end{equation*}
    with probability at least $1 - 2\exp(-t^2)$.
\label{lem:random_matrix}
\end{lemma}

Using the above lemma, we can prove the following lemma:
\begin{lemma}
   Let $\mX$ be the randomly generated data matrix $\mX = (\vx_1,\cdots,\vx_n)$, with $\vx_i = (x_{i1},\cdots, x_{id}) \sim \gN(0, \mI)$ for $i\in[n]$. For sufficiently large $d$ and $n$,  we have 
    \begin{equation*}
        C(\vw_\gM) = \dfrac{1}{n} \left(\tr(\mX^\top \mX) - (\gamma+2)\|\mX^\top \overline{\vw_\gM}\|^2\right) \geq \epsilon d,
    \end{equation*} 
    for all $\vw_\gM\in\gM$, with probability at least $1 - \exp(-\Omega(nd))$, where $\epsilon$ is an arbitrary small positive constant.
\label{lem:positive_C}
\end{lemma} 
\begin{proof}

    By Lemma~\ref{lem:random_matrix}, we have 
    \begin{equation*}
        \|\mX^\top \overline{\vw_\gM}\|^2 \leq \|\mX^\top \|^2 \| \overline{\vw_\gM}\|^2 = \|\mX\|^2 \leq \gO((\sqrt{n} + \sqrt{d} + t)^2),
    \end{equation*}
    with probability at least $1 - 2\exp(-t^2)$. We take $t = \sqrt{d}$ and obtain
    \begin{equation*}
        \|\mX^\top \overline{\vw_\gM}\|^2 \leq \gO(d), \quad \text{with probability at least} \quad 1 - 2\exp(-d^2).
    \end{equation*}

    By the central limit theorem, 
    \begin{equation*}
        \sqrt{nd}\left[\dfrac{\tr(\mX^\top \mX)}{nd} - 1\right] = \sqrt{nd}\left[\dfrac{\sum_{i=1}^n \sum_{j=1}^d x_{ij}^2}{nd} - 1\right] \xrightarrow{D} \gN\left(0, 2\right),
    \end{equation*}
    and thus for any $K>0$, $\tr(\mX^\top \mX) \geq nd/2$ with probability at least 
    \begin{equation*}
        1 - \gP(\tr(\mX^\top \mX) \leq nd/2) \to 1 - \Phi\left(\dfrac{nd/2 - nd}{\sqrt{2nd}}\right) \geq 1 - \Phi(-\Omega(\sqrt{nd})) \geq 1 -  \exp(-\Omega(nd)).
     \end{equation*}

    Therefore, with a probability of at least 
    \begin{equation*}
        1 - \exp(-\Omega(nd)) - 2\exp(-d^2) \geq 1 - \exp(-\Omega(nd)),
    \end{equation*}
    we have 
    \begin{equation*}
        C(\vw_\gM) = \dfrac{1}{n}\left(\tr(\mX^\top \mX) - (\gamma+2)\|\mX^\top \overline{\vw_\gM}\|^2\right) \geq \dfrac{1}{n}\left(\dfrac{nd}{2} - (\gamma+2)\gO(d)\right) \geq \Omega(d),
    \end{equation*}
    
    The statement in the lemma follows.
\end{proof}

We are now ready to prove Theorem~\ref{thm:phase_II}:
\begin{proof}[Proof of Theorem~\ref{thm:phase_II}]
    By Lemma~\ref{lem:positive_C}, we have with probability at least $1 - \exp(-\Omega(nd))$, the data matrix $\mX$ and the ground truth $\vw^*$ permits the existence of following constant
    \begin{equation*}
        \underline{C} = \inf_{t\geq 0} \dfrac{C(\vw_\gM(t))}{d} > 0,
    \end{equation*}

    The following ODE of $\|\vw_\gM(t)\|$ is directly computed from~\Eqref{eq:2_effective_dynamics} from Lemma~\ref{lem:effective_dynamics}:
    \begin{equation}
        \begin{aligned}
            \dfrac{\rd \|\vw_\gM(t)\|}{\rd t} =& -\dfrac{\sigma^2 \eta_L\gamma}{2}\|\vw_\gM(t)\|^{2\gamma-1} C(\vw_\gM(t))\langle \overline \vw_\gM(t), \gP_{\mA_\gM^{-1}\mX^\perp}\overline{\vw_\gM(t)}\rangle\\
            \leq & -\dfrac{\sigma^2 \eta_L\gamma\underline{C}d}{2}\|\vw_\gM(t)\|^{2\gamma-1}  \langle \overline \vw_\gM(t), \gP_{\mA_\gM^{-1}\mX^\perp}\overline{\vw_\gM(t)}\rangle.
        \end{aligned}
    \label{eq:2_ode}
    \end{equation}

    Define $\vv(t) = (\mX^\perp)^\top \overline{\vw_\gM(t)}$, by~\Eqref{eq:sherman_morrison}, we have 
    \begin{equation*}
        \begin{aligned}
            ((\mX^\perp)^\top\mA_\gM^{-2}\mX^\perp)^{-1} &= \left(\left(\mX^\perp\right)^\top \|\vw_\mX(t)\|^{-2\gamma} \left(\mI - \dfrac{\gamma}{1+\gamma} \overline{\vw_\gM(t)}\ \overline{\vw_\gM(t)} ^\top\right)^2 \mX^\perp\right)^{-1}\\
            &=\|\vw_\mX(t)\|^{2\gamma}  \left(\left(\mX^\perp\right)^\top \left(\mI - \dfrac{\gamma^2 + 2\gamma}{(1+\gamma)^2} \overline{\vw_\gM(t)}\ \overline{\vw_\gM(t)} ^\top\right) \mX^\perp\right)^{-1}\\
            &=\|\vw_\mX(t)\|^{2\gamma}  \left(\mI - \dfrac{\gamma^2 + 2\gamma}{(1+\gamma)^2} \vv(t) \vv(t)^\top \right)^{-1}\\
            &=\|\vw_\mX(t)\|^{2\gamma} \left(\mI +\dfrac{ \frac{\gamma^2 + 2\gamma}{(1+\gamma)^2} \vv(t) \vv(t)^\top}{1 + \frac{\gamma^2 + 2\gamma}{(1+\gamma)^2} \|\vv(t)\|^2 }\right),
        \end{aligned}
    \end{equation*}
    where the last equality is due to the Sherman-Morrison formula. 
    
    Then we compute as follows:
    \begin{equation*}
        \begin{aligned}
            \langle \overline \vw_\gM(t), \gP_{\mA_\gM^{-1}\mX^\perp}\overline{\vw_\gM(t)}\rangle&= \langle \overline{\vw_\gM(t)}, \mA_\gM^{-1}\mX^\perp ((\mX^\perp)^\top\mA_\gM^{-2}\mX^\perp)^{-1} (\mX^\perp)^\top \mA_\gM^{-1}\overline{\vw_\gM(t)}\rangle\\
            &= \dfrac{\|\vw_\gM(t)\|^{-2\gamma}}{(1+\gamma)^2} \langle \overline{\vw_\gM(t)},\mX^\perp ((\mX^\perp)^\top\mA_\gM^{-2}\mX^\perp)^{-1} (\mX^\perp)^\top \overline{\vw_\gM(t)}\rangle,\\
            &= \dfrac{\|\vw_\gM(t)\|^{-2\gamma}}{(1+\gamma)^2} \langle \vv(t),   \|\vw_\mX(t)\|^{2\gamma}  \left(\mI +\dfrac{ \frac{\gamma^2 + 2\gamma}{(1+\gamma)^2} \vv(t) \vv(t)^\top}{1 + \frac{\gamma^2 + 2\gamma}{(1+\gamma)^2} \|\vv(t)\|^2 }\right)  \vv(t)\rangle,\\
            &= \dfrac{1}{(1+\gamma)^2} \|\vv(t)\|^2 \left(1 + \dfrac{ \frac{\gamma^2 + 2\gamma}{(1+\gamma)^2}  \|\vv(t)\|^2}{1 + \frac{\gamma^2 + 2\gamma}{(1+\gamma)^2} \|\vv(t)\|^2 }  \right) \geq \dfrac{1}{(1+\gamma)^2} \|\vv(t)\|^2,
        \end{aligned}
    \end{equation*}
    where the second equality is by~\Eqref{eq:sherman_morrison}.

    Plugging into~\Eqref{eq:2_ode}, we have
    \begin{equation}
        \dfrac{\rd \|\vw_\gM(t)\|}{\rd t} \leq -\dfrac{\sigma^2 \eta_L\gamma\underline{C}d}{2(1+\gamma)^2}\|\vw_\gM(t)\|^{2\gamma-1} \|\vv(t)\|^2.
    \label{eq:2_ode_2}
    \end{equation}

    We consider the representation of $\vw_\gM(t)$ in Proposition~\ref{prop:representation} as $\lambda(t) = \|\vw_\gM(t)\|^{-\gamma}$, by which we have 
    \begin{equation*}
        \begin{aligned}
            \|\vv(t)\|^2 &= \overline{\vw_\gM(t)}^\top \mX^\perp (\mX^\perp)^\top \overline{\vw_\gM(t)} = \overline{\vw_\gM(t)}^\top \mX^\perp (\mX^\perp)^\top  \mX^\perp (\mX^\perp)^\top \overline{\vw_\gM(t)} \\
            &= \|\gP_{\mX^\perp} \overline{\vw_\gM(t)}\|^2 = \dfrac{\|\gP_{\mX^\perp}\vw_\gM(t)\|^2 }{\|\vw_\gM(t)\|^2} \\
            &= \dfrac{\lambda(t)^{-2/\gamma} - \lambda(t)^2 \|\valpha_\mX^*\|^2}{\lambda(t)^{-2/\gamma}} = 1 - \lambda(t)^{2+2/\gamma} \|\valpha_\mX^*\|^2\\
            &= 1 - \|\vw_\gM(t)\|^{-2-2\gamma} \|\valpha_\mX^*\|^2.
        \end{aligned}
    \end{equation*}

    Plugging into~\Eqref{eq:2_ode_2}, we have
    \begin{equation*}
        \dfrac{\rd \|\vw_\gM(t)\|}{\rd t} \leq -\dfrac{\sigma^2 \eta_L\gamma\underline{C}d}{2(1+\gamma)^2}\left(\|\vw_\gM(t)\|^{2\gamma-1}  - \|\vw_\gM(t)\|^{-3} \|\valpha_\mX^*\|^2\right).
    \end{equation*}

    For any $\gamma > -1$, since $\vw^\dagger$ is the unique minimizer, we have $\|\vw_\gM(t)\|\geq  \|\valpha_\mX^*\|^{\frac{1}{1+\gamma}}$, and thus 
    \begin{equation*}
        \|\vw_\gM(t)\|^{2\gamma-1}  \geq  \|\vw_\gM(t)\|^{-3} \|\valpha_\mX^*\|^2,
    \end{equation*}
    where the equality holds if and only if $\vw_\gM(t) = \vw^\dagger$. Consequently, $\|\vw_\gM(t)\|$ converges to $\|\vw^\dagger\|$.

    For any $\gamma > 1/2$, by taking $\delta(t) =  \|\vw(t)\| - \|\vw^\dagger\| =  \|\vw(t)\| - \|\valpha_\mX^*\|^{\frac{1}{1+\gamma}}$, we have
    \begin{equation*}
       \begin{aligned}
        \dfrac{\rd \delta(t)}{\rd t} \leq& -\dfrac{\sigma^2 \eta_L\gamma\underline{C}d}{2(1+\gamma)^2}\left( \|\valpha_\mX^*\|^{\frac{2\gamma-1}{1+\gamma}} 
        +(2\gamma -1 )\|\valpha_\mX^*\|^{\frac{2\gamma-2}{1+\gamma} }\delta(t) - \|\valpha_\mX^*\|^{\frac{-3}{1+\gamma}} \|\valpha_\mX^*\|^2\right)\\
        \leq& -\sigma^2 \eta_L\underline{C}d\dfrac{\gamma(2\gamma - 1)}{2(\gamma+1)^2}\delta(t),
       \end{aligned}
    \end{equation*}
    solving which yields
    \begin{equation}
        \delta(t) \leq \delta(0) \exp\left(-\sigma^2 \eta_L\underline{C}d\dfrac{\gamma(2\gamma - 1)}{2(\gamma+1)^2}t\right).
    \label{eq:2_delta}
    \end{equation}
    And the statement in the theorem follows.
\end{proof}

\begin{remark}
    Besides Remark~\ref{rem:phase_II} for Theorem~\ref{thm:phase_II} in the main text, one should also notice that although the convergence of $\|\vw_\gM(t)\|$ to $\|\vw^\dagger\|$ depends on the parameter $\gamma$, representing the depth of the neural network, in~\Eqref{eq:2_delta}, its convergence rate has an upper bound as $\gamma\to\infty$. 
\end{remark}

\section{MISSING PROOF FOR PHASE III}
\label{app:phase_III}

In this section, we prove Theorem~\ref{thm:phase_III}. For convenience, we will assume the time $t$ is reset to $0$ at the beginning of Phase III, and we suppose $\vw(0)$ is near a point $\vw_\gM$ on the minima manifold $\gM$. Specifically, we consider the projection of the dynamics $\vw(t)$ onto the tangent space $\gT(\vw_\gM;\gM)$ and the normal space $\gN(\vw_\gM;\gM)$ of the minima manifold $\gM$ around $\vw_\gM$, \emph{i.e.}
\begin{equation*}
    \vw(t) = \vw_\gM + \Delta \vw_\parallel(t) + \Delta \vw_\perp(t),
\end{equation*}

As argued in Section~\ref{sec:phase_II}, the dynamics of $\Delta \vw_\perp(t)$ are faster than that of $\Delta \vw_\parallel(t)$ by a factor of $\Theta(\eta_L)$. Thus, we assume $\Delta \vw_\parallel(t)\equiv 0$ during Phase III and focus on the dynamics of $\Delta \vw_\perp(t)$ (or equivalently rescale the time $t$ with $\tau$ of a smaller scale as we did for Phase II). Also, it is straightforward to see that $\Delta \vw_\perp(t)\in \mA_\gM \mX$ by Proposition~\ref{prop:space}. 

\begin{proof}[Proof of Theorem~\ref{thm:phase_III}]
    We approximate the landscape $\hat{\gL}(\vw(t); \vw^*, \gD)$ by its second-order Taylor expansion around $\vw_\gM$:
    \begin{equation*}
        \hat{\gL}(\vw(t); \vw^*, \gD) \approx \dfrac{1}{2} (\vw - \vw_\gM)^\top \nabla^2 \hat{\gL}(\vw_\gM; \vw^*, \gD) (\vw - \vw_\gM),
    \end{equation*}
    due to 
    \begin{equation*}
        \hat{\gL}(\vw_\gM; \vw^*, \gD) =0 \quad\text{and}\quad \nabla \hat{\gL}(\vw_\gM; \vw^*, \gD) = \vzero.
    \end{equation*}
    And the dynamics of $\Delta \vw_\perp(t)$ are thus approximated by 
    \begin{equation*}
        \dfrac{\rd \Delta \vw_\perp(t)}{\rd t} = -\nabla^2 \hat{\gL}(\vw_\gM; \vw^*, \gD) \Delta \vw_\perp(t) = - \dfrac{1}{n}\mA_\gM\mX\mX^\top \mA_\gM\Delta\vw_\perp(t).
    \end{equation*}

    Since $\Delta \vw_\perp(t)\in \mA_\gM \mX$, let $\Delta \vw_\perp(t) = \mA_\gM \mX \veps(t)$, we have 
    \begin{equation*}
        \mA_\gM \mX \dfrac{\rd \veps(t)}{\rd t} = -\dfrac{1}{n} \mA_\gM \mX \left(\mX^\top \mA_\gM^2 \mX \right)\veps(t),\quad \text{\emph{i.e.}}\quad \dfrac{\rd \veps(t)}{\rd t} = -\dfrac{1}{n} \left(\mX^\top \mA_\gM^2 \mX \right)\veps(t).
    \end{equation*}
 
    Then we have 
    \begin{equation*}
        \dfrac{\rd \|\veps(t)\|}{\rd t} = -\dfrac{1}{n} \langle \overline{\veps(t)},(\mX^\top (\mA_\gM)^2\mX) \veps(t)\rangle \leq -\dfrac{\lambda_{\min}(\mX^\top (\mA_\gM)^2\mX)}{n}\|\veps(t)\|,
    \end{equation*}
    where $\lambda_{\min}(\mX^\top (\mA_\gM)^2\mX)$ is the smallest eigenvalue of $\mX^\top (\mA_\gM)^2\mX$, which is guaranteed to be positive, given $\mX$ is full rank and
    \begin{equation*}
        \mX^\top (\mA_\gM)^2\mX =\|\vw_\gM\|^{2\gamma} \mX^\top \left(\mI + \gamma \overline{\vw_\gM}\ \overline{\vw_\gM}^\top\right)^2 \mX 
    \end{equation*}
    is positive definite. 
    Thus, we have the exponential convergence 
    \begin{equation*}
        \|\veps(t)\| \leq \|\veps(0)\| \exp\left(-\dfrac{\lambda_{\min}(\mX^\top (\mA_\gM)^2\mX)}{n}t\right),
    \end{equation*}
    and thus
    \begin{equation*}
        \|\Delta\vw_\perp(t)\| \leq \|\mA_\gM \mX\| \|\veps(t)\| \leq \|\mA_\gM\|\|\mX\|\|\veps(0)\| \exp\left(-\dfrac{\lambda_{\min}(\mX^\top (\mA^\dagger)^2\mX)}{n}t\right).
    \end{equation*}
    And the statement in the theorem follows. 
\end{proof}

\end{document}